\documentclass[reqno]{amsart}
\usepackage[margin=1.5in,bottom=1.25in]{geometry}	% for tighter margins (80 CPL)

\usepackage{amsmath}	% for AMS macros
\usepackage{amssymb}	% for AMS symbols
\usepackage{amsfonts}	% for AMS fonts
\usepackage{amsthm}	% for theorems

\usepackage[centercolon=true]{mathtools}	% for advanced math

\mathtoolsset{%
}

\usepackage[utf8]{inputenc}	% for source encoding
\usepackage[T1]{fontenc}	% for font encoding

\usepackage[%	% for math font selection
cal=cm,
]
{mathalfa}

\usepackage{dsfont}	% for blackboard bold font
\usepackage[light,scaled=.95]{AlegreyaSans}
\usepackage[proportional,tabular,lining,sf,mono=false]{libertine}
\usepackage{acronym}	% for acronyms
\newcommand{\acli}[1]{\emph{\acl{#1}}}	% for italicized acro
\newcommand{\acdef}[1]{\define{\acl{#1}} \textup{(\acs{#1})}\acused{#1}}	% for acro def
\usepackage[labelfont={bf,small},labelsep=colon,font=small]{caption}	% for caption control
\usepackage{subcaption}	% for subfigures
\usepackage[svgnames]{xcolor}	% for color
\colorlet{MyRed}{Crimson!60!DarkRed}
\colorlet{MyBlue}{DodgerBlue!75!black}
\colorlet{MyGreen}{DarkGreen}
\colorlet{MyViolet}{DarkMagenta}

\colorlet{MyLightBlue}{DodgerBlue!20}
\colorlet{MyLightGreen}{MyGreen!20}

\colorlet{PrimalColor}{MyBlue}
\colorlet{PrimalFill}{MyLightBlue}
\colorlet{DualColor}{MyRed}

\colorlet{AlertColor}{MyRed}	% for alerts
\colorlet{BadColor}{MyRed}	% for good
\colorlet{GoodColor}{MyGreen}	% for good
\colorlet{LinkColor}{MediumBlue}	% for hyperlinks
\colorlet{MacroColor}{MyViolet}
\colorlet{RevColor}{MediumBlue}	% for revisions

\usepackage[]{titlesec}	% for section headings (may create problem with appendices)
\titleformat{name=\section}{\medskip}{\thetitle.}{0.8em}{\centering\scshape}

\titleformat{name=\subsection}[runin]{}{\bfseries\thetitle.}{0.5em}{\bfseries}[.]
\titleformat{name=\subsubsection}[runin]{}{\thetitle.}{0.5em}{\bfseries}[.]

\titleformat{name=\paragraph,numberless}[runin]{}{}{0em}{\bfseries}[]
\titlespacing{\paragraph}{0em}{\medskipamount}{1em}

\titleformat{name=\subparagraph,numberless}[runin]{}{}{0em}{}[.]
\titlespacing{\subparagraph}{0em}{0em}{0.5em}

\newcommand{\afterhead}{.\;}	% for changing headings
\newcommand{\para}[1]{\paragraph{\textbf{#1\afterhead}}}	% for paragraph control

\usepackage{latexsym}	% for symbols
\usepackage{fontawesome}	% for icons
\usepackage{pifont}	% for dingbats

\usepackage{tikz}	% for figures
\usepackage{tikz-cd}	% for diagrams
\usetikzlibrary{calc,patterns}	% for basic tikz figures

\usepackage{array}	% for flexible tables
\usepackage{booktabs}	% for better tables
\usepackage[inline,shortlabels]{enumitem}	% for lists
\setlist[1]{topsep=\smallskipamount,itemsep=\smallskipamount,left=\parindent}
\setlist[2]{left=0pt}
\usepackage[kerning=true]{microtype}	% for microtypography

\usepackage{lipsum}	% for placeholder text
\usepackage{tabto}	% for spacing
\usepackage{xspace}	% for flexible spaces

\usepackage[sort&compress]{natbib}	% for citations

\setcitestyle{numbers,square}

\usepackage{hyperref}
\hypersetup{
final,
colorlinks=true,
linktocpage=true,
pdfstartview=FitH,
breaklinks=true,
pdfpagemode=UseNone,
pageanchor=true,
pdfpagemode=UseOutlines,
plainpages=false,
bookmarksnumbered,
bookmarksopen=false,
bookmarksopenlevel=1,
hypertexnames=true,
pdfhighlight=/O,
hyperfootnotes=false,
urlcolor=LinkColor,linkcolor=LinkColor,citecolor=LinkColor,	% for on-screen
pdftitle={},
pdfauthor={},
pdfsubject={},
pdfkeywords={},
pdfcreator={pdfLaTeX},
pdfproducer={LaTeX with hyperref}
}

\newcommand{\EMAIL}[1]{\email{\href{mailto:#1}{#1}}}

\usepackage[sort&compress,capitalize,nameinlink]{cleveref}	% for cleveref formatting

\crefname{algo}{Algorithm}{Algorithms}
\crefname{assumption}{Assumption}{Assumptions}
\usepackage{algorithm}	% for algorithm environments
\usepackage{algpseudocode}	% for algorithm macros
\usepackage{thmtools}	% for theorem tools
\usepackage{thm-restate}	% for restating theorems

\theoremstyle{plain}
\newtheorem{theorem}{Theorem}	% for theorems
\newtheorem{corollary}{Corollary}	% for corollaries
\newtheorem{lemma}{Lemma}	% for lemmas
\newtheorem{proposition}{Proposition}	% for propositions

\newtheorem*{theorem*}{Theorem}	% for theorems
\newtheorem*{corollary*}{Corollary}	% for corollaries (unnumbered)

\theoremstyle{definition}
\newtheorem{assumption}{Assumption}	% for assumptions
\newtheorem*{definition*}{Definition}	% for definitions (unnumbered)
\newtheorem*{assumption*}{Assumptions}	% for assumptions (unnumbered)
\newtheorem*{example*}{Example}	% for examples (unnumbered)

\theoremstyle{remark}
\newtheorem*{remark*}{Remark}	% for remarks (unnumbered)
\newtheorem*{notation*}{Notation}	% for notation (unnumbered)

\newcounter{proofstep}

\numberwithin{remark}{section}	% for remark numbering
\numberwithin{example}{section}	% for example numbering

\usepackage[showdeletions]{color-edits}	% for editing macros / use [suppress] for final
\newcommand{\draft}[1]{#1}	% for removing macro coloring

\newcommand{\define}[1]{\emph{\draft{#1}}}	% for revision markup
\newcommand{\newmacro}[2]{\newcommand{#1}{\draft{#2}}}	% for semantic definitions
\newcommand{\newop}[2]{\DeclareMathOperator{#1}{\draft{#2}}}	% for semantic definitions
\newcommand{\newopstar}[2]{\DeclareMathOperator*{#1}{\draft{#2}}}	% for semantic definitions

\newcommand{\eps}{\varepsilon}	% for better epsilon
\DeclarePairedDelimiter{\parens}{(}{)}	% for parentheses
\DeclarePairedDelimiter{\abs}{\lvert}{\rvert}	% for absolute value
\DeclarePairedDelimiterX{\setdef}[2]{\{}{\}}{#1:#2}	% for set builder notation
\DeclarePairedDelimiterXPP{\exclude}[1]{\mathopen{}\setminus}{\{}{\}}{}{#1}

\DeclarePairedDelimiterX{\braket}[2]{\langle}{\rangle}{#1\mathopen{}\delimsize\vert\mathopen{}#2}
\DeclarePairedDelimiterX{\inner}[2]{\langle}{\rangle}{#1,#2}	% for scalar product

\DeclarePairedDelimiter{\norm}{\lVert}{\rVert}	% for norm
\DeclarePairedDelimiterXPP{\dnorm}[1]{}{\lVert}{\rVert}{_{\ast}}{#1}	% for dual norm
\DeclarePairedDelimiterXPP{\onenorm}[1]{}{\lVert}{\rVert}{_{1}}{#1}	% for L1 norm
\DeclarePairedDelimiterXPP{\twonorm}[1]{}{\lVert}{\rVert}{_{2}}{#1}	% for L2 norm
\DeclarePairedDelimiterXPP{\supnorm}[1]{}{\lVert}{\rVert}{_{\infty}}{#1}	% for sup norm
\DeclarePairedDelimiterXPP{\frobnorm}[1]{}{\lVert}{\rVert}{_{F}}{#1}	% for sup norm
\DeclarePairedDelimiterXPP{\tvnorm}[1]{}{\lVert}{\rVert}{_{\mathrm{TV}}}{#1}	% for sup norm

\newcommand{\defeq}{\coloneqq}	% for direct definition
\newcommand{\eqdef}{\eqqcolon}	% for reverse definition

\newcommand{\from}{\colon}	% for function definition
\newmacro{\F}{\mathbb{F}}	% generic field
\newmacro{\N}{\mathbb{N}}	% for naturals
\newmacro{\Z}{\mathbb{Z}}	% for integers
\newmacro{\Q}{\mathbb{Q}}	% for rationals

\newmacro{\real}{x}	% for generic real
\newmacro{\reals}{\mathbb{R}}	% for set of reals
\newmacro{\R}{\reals}	% for reals

\newmacro{\complex}{z}	% for generic complex
\newmacro{\complexes}{\mathbb{C}}	% for set of reals
\newmacro{\C}{\complexes}	% for complex numbers (may clash)

\newopstar{\argmax}{arg\,max}	% for argmax
\newopstar{\argmin}{arg\,min}	% for argmin
\newopstar{\intersect}{\bigcap}	% for intersections
\newopstar{\union}{\bigcup}	% for unions

\newop{\aff}{aff}	% for affine hull
\newop{\bd}{bd}	% for boundary
\newop{\bigoh}{\mathcal{O}}	% for Landau O
\newop{\card}{card}	% for cardinality
\newop{\cl}{cl}	% for closure
\newop{\conv}{conv}	% for convex hull (but see also \simplex)
\newop{\crit}{crit}	% for critical
\newop{\curl}{curl}	% for curl
\newop{\diag}{diag}	% for diagonal matrices
\newop{\diam}{diam}	% for diameter
\newop{\dist}{dist}	% for distance
\newop{\diver}{div}	% for divergece
\newop{\dom}{dom}	% for domain
\newop{\eig}{eig}	% for eigenvalues
\newop{\ess}{ess}	% for essential
\newop{\grad}{grad}	% for gradient
\newop{\Hess}{Hess}	% for Hessian
\newop{\ind}{ind}	% for index
\newop{\im}{im}	% for image
\newop{\intr}{int}	% for interior
\newop{\Jac}{Jac}	% for Jacobian
\newop{\one}{\mathds{1}}	% for indicator
\newop{\proj}{pr}	% for projection
\newop{\prox}{prox}	% for prox
\newop{\rank}{rank}	% for rank
\newop{\relint}{ri}	% for relative interior
\newop{\sign}{sgn}	% for sign
\newop{\supp}{supp}	% for support
\newop{\Sym}{Sym}	% for symmetric
\newop{\tr}{tr}	% for trace
\newop{\unif}{unif}	% for uniform
\newop{\vol}{vol}	% for volume

\newcommand{\cf}{cf.\xspace}	% for consistency
\newcommand{\ie}{i.e.,\xspace}	% for consistency
\newcommand{\vs}{vs.\xspace}	% for consistency
\newcommand{\textpar}[1]{\textup(#1\textup)}	% for upshape parentheses

\newcommand{\alt}[1]{#1'}		% for alternate version
\newcommand{\altalt}[1]{#1''}		% for second alternate
\newmacro{\ball}{\mathbb{B}}	% for ball
\newmacro{\sphere}{\mathbb{S}}	% for sphere

\newmacro{\argdot}{\boldsymbol{\cdot}}	% for argument
\newmacro{\dd}{\:d}	% for integration
\newmacro{\ddt}{\frac{d}{dt}}	% for Leibniz
\newmacro{\del}{\partial}	% for derivatives

\newmacro{\const}{c}	% for generic constant
\newmacro{\Const}{C}	% for generic constant

\newmacro{\param}{\theta}	% for parameter
\newmacro{\params}{\Theta}	% for set of parameters

\newmacro{\coef}{\lambda}	% for generic coefficient

\newmacro{\fn}{f} % for generic function

\newmacro{\pexp}{p}	% for first exponent
\newmacro{\qexp}{q}	% for second exponent
\newmacro{\rexp}{r}	% for third exponent

\newmacro{\idx}{i}
\newmacro{\idxalt}{j}
\newmacro{\idxaltalt}{k}
\newmacro{\nIndices}{I}
\newmacro{\indices}{\mathcal{I}}

\newmacro{\point}{x}	% for generic point
\newmacro{\pointalt}{\alt\point}	% for second point
\newmacro{\pointaltalt}{\altalt\point}	% for third point
\newmacro{\points}{\mathcal{X}}	% for set of points
\newmacro{\intpoints}{\relint\points}	%for interior points

\newmacro{\base}{p}	% for reference point
\newmacro{\basealt}{q}	% for second reference point
\newmacro{\basealtalt}{u}	% for third reference point

\newmacro{\set}{\mathcal{S}}	% for generic set

\newmacro{\borel}{\mathcal{B}}	% for Borel sets
\newmacro{\closed}{\mathcal{C}}	% for closed sets
\newmacro{\cpt}{\mathcal{K}}	% for compact sets
\newmacro{\nhd}{\mathcal{U}}	% for neighborhoods
\newmacro{\open}{\mathcal{U}}	% for open sets

\newmacro{\domain}{\mathcal{D}}	% for domain
\newmacro{\region}{\mathcal{R}}	% for region

\newmacro{\interval}{\mathcal{I}}	% for intervals
\newmacro{\rectangle}{\mathcal{R}}	% for rectangles

\newmacro{\vecspace}{\mathcal{V}}	% for vector space
\newmacro{\subspace}{\mathcal{Z}}	% for subspace
\newmacro{\dualspace}{\pspace^{\ast}}	% for dual space
\newmacro{\vdim}{d}	% dimension

\newmacro{\unitvec}{u}	% for unit vector
\newmacro{\bvec}{e}	% for basis vector
\newmacro{\bvecs}{\mathcal{E}}	% for basis vectors

\newmacro{\pvec}{z}	% for primal direction
\newmacro{\pvecalt}{\alt\pvec}	% for second primal direction
\newmacro{\pvecaltalt}{\altalt\pvec}	% for third primal direction
\newmacro{\pvecs}{\vecspace}	% for set of primal directions
\newmacro{\pspace}{\pvecs}	% alias

\newmacro{\dvec}{w}	% for dual vector
\newmacro{\dvecalt}{\alt\pvec}	% for second dual vector
\newmacro{\dvecaltalt}{\altalt\pvec}	% for third dual vector
\newmacro{\dvecs}{\dualspace}	% for set of dual vectors
\newmacro{\dspace}{\dvecs}	% alias

\newmacro{\coord}{i}	% for coordinate index
\newmacro{\coordalt}{j}	% for second coordinate
\newmacro{\coordaltalt}{k}	% for third coordinate
\newmacro{\nCoords}{d}	% for dimension

\newmacro{\vecfield}{v}	% for vector field
\newmacro{\vbound}{V}	% for vector bound

\newmacro{\cvx}{\mathcal{C}}	% for generic convex set

\newmacro{\subd}{\partial}	% for subdifferential
\newmacro{\subsel}{\nabla}	% for subdifferential selection

\newop{\tcone}{TC}	% for tangent cone
\newop{\dcone}{\tcone^{\ast}}	% for dual cone
\newop{\ncone}{NC}	% for normal cone
\newop{\pcone}{PC}	% for polar cone
\newop{\hull}{\Delta}	% for convex hull

\newop{\Opt}{\mathsf{Opt}}	% for value of problem
\newop{\Sol}{\mathsf{Sol}}	% for solution of problem
\newop{\gap}{\mathsf{Gap}}	% for gap function
\newop{\orcl}{\mathsf{G}}	% for oracle

\newmacro{\obj}{f}	% for objective function
\newmacro{\sobj}{F}	% for stochastic objective

\newmacro{\gvec}{g}	% for gradient vector
\newmacro{\oper}{A}	% for operator

\newmacro{\lips}{L}	% for Lipschitz modulus
\newmacro{\gbound}{G}	% for gradient bound
\newmacro{\strong}{\alpha}	% for strong convexity modulus
\newmacro{\smooth}{\beta}	% for strong smoothness modulus

\newmacro{\radius}{r}
\newmacro{\Radius}{R}

\newmacro{\mfld}{\mathcal{M}}	% for manifold

\newmacro{\gmat}{g}	% for metric tensor
\newmacro{\gdist}{\dist_{\gmat}}	% for metric distance

\newmacro{\tanvec}{z}	% for tangent vector
\newmacro{\form}{\omega}	% for generic form

\newop{\ex}{\mathbb{E}}	% for expectation
\newop{\prob}{\mathbb{P}}	% for probability
\newop{\Var}{\mathbb{V}}	% for variance
\newop{\simplex}{\Delta}	% for simplices

\DeclarePairedDelimiterXPP{\exof}[1]{\ex}{[}{]}{}{%	% for conditional expectations
 #1}

\DeclarePairedDelimiterXPP{\exwrt}[2]{\ex_{#1}}{[}{]}{}{%		% for conditional expectations
 #2}

\DeclarePairedDelimiterXPP{\probof}[1]{\prob}{(}{)}{}{%	% for conditional probabilities
 #1}

\DeclarePairedDelimiterXPP{\probwrt}[2]{\prob_{#1}}{[}{]}{}{%		% for conditional expectations
 #2}

\DeclarePairedDelimiterXPP{\oneof}[1]{\one}{\{}{\}}{}{#1}	% for indicators

\newmacro{\rv}{X}	% for random variable
\newmacro{\rvalt}{Y}	% for second random variable

\newmacro{\event}{E}       % for generic event
\newmacro{\eventalt}{H}       % for second event

\newmacro{\seed}{\theta}	% for seed
\newmacro{\seeds}{\Theta}	% for seed space
\newmacro{\pdist}{P}	% for seed law
\newmacro{\history}{\mathcal{H}}	% for filtrations

\newmacro{\sample}{\omega}	% for sample
\newmacro{\samples}{\Omega}	% for set of samples

\newmacro{\filter}{\mathcal{F}}	% for filtration
\newmacro{\probspace}{(\samples,\filter,\prob)}	% for probability space

\newmacro{\mean}{\mu}	% for mean
\newmacro{\sdev}{\sigma}	% for standard deviation
\newmacro{\variance}{\sdev^{2}}	% for variance
\newmacro{\covmat}{\Sigma}	% for covariance matrix

\newmacro{\seq}{a}	% for sequence
\newmacro{\seqalt}{b}	% for second sequence
\newmacro{\seqaltalt}{c}	% for third sequence

\newmacro{\beforestart}{0}	% for before start index
\newmacro{\start}{1}	% for start index
\newmacro{\afterstart}{2}	% for second index
\newmacro{\running}{\start,\afterstart,\dotsc}	% for running

\newmacro{\run}{n}	% for generic index
\newmacro{\runalt}{k}	% for second index
\newmacro{\runaltalt}{\ell}	% for third index
\newmacro{\nRuns}{T}	% for index size
\newmacro{\runs}{\mathcal{\nRuns}}	% for index set

\newmacro{\curve}{\gamma}	% for generic curve
\DeclarePairedDelimiterXPP{\curveof}[1]{\curve}{(}{)}{}{#1}	% for curve
\DeclarePairedDelimiterXPP{\dotcurveof}[1]{\dot\curve}{(}{)}{}{#1}	% for velocity of curve

\newmacro{\tstart}{0}	% for time start
\newmacro{\timealt}{s}	% for second time
\newmacro{\timealtalt}{\tau}	% for third time
\newmacro{\horizon}{T}	% for horizon
\newmacro{\tend}{\horizon}	% for horizon
\newmacro{\window}{[\tstart,\tend]}	% for window

\newmacro{\state}{x}	% for generic state variable
\newmacro{\statealt}{y}	% for second state variable
\newmacro{\statealtalt}{z}	% for third state variable

\newmacro{\mat}{M}	% for generic matrix
\newmacro{\hmat}{H}	% for Hessian matrix

\newmacro{\ones}{\mathbf{1}}	% for matrix of ones
\newmacro{\eye}{I}	% for identity matrix
\newmacro{\zer}{\mathbf{0}}	% for zero matrix

\newmacro{\eigval}{\lambda}	% for eigenvalue
\newmacro{\eigvec}{u}	% for eigenvector

\newop{\Nash}{NE}	% for Nash equilibria
\newop{\CE}{CE}	% for correlated equilibria
\newop{\CCE}{CCE}	% for Hannan set
\newop{\NI}{NI}	% for Nikaido-Isoda function

\newop{\brep}{br}	% for best responses
\newop{\preg}{\overline{Reg}}	% for pseudo-regret
\newop{\val}{val}	% for value function

\newmacro{\play}{i}	% for generic player
\newmacro{\playalt}{j}	% for second player
\newmacro{\playaltalt}{k}	% for third player
\newmacro{\nPlayers}{N}	% for number of players
\newmacro{\players}{\mathcal{\nPlayers}}	% for set of players

\newmacro{\pure}{\alpha}	% for generic pure
\newmacro{\purealt}{\beta}	% for second pure
\newmacro{\purealtalt}{\gamma}	% for third pure
\newmacro{\nPures}{A}	% for number of pures
\newmacro{\pures}{\mathcal{\nPures}}	% for set of pures

\newmacro{\strat}{x}	% for generic strategy
\newmacro{\stratalt}{\alt\strat}	% for second strategy
\newmacro{\strataltalt}{\altalt\strat}	% for third strategy
\newmacro{\strats}{\mathcal{X}}	% for set of strategies
\newmacro{\intstrats}{\strats^{\circle}}	% for set of interior strategies

\newmacro{\pay}{u}	% for payoff function
\newmacro{\loss}{\ell}	% for loss function

\newmacro{\payv}{v}	% for payoff vector
\newmacro{\payfield}{v}	% for payoff field

\newmacro{\game}{\mathcal{G}}	% for game
\newmacro{\gamefull}{\game(\players,\points,\pay)}	% for game, full description

\newmacro{\fingame}{\Gamma}	% for finite game
\newmacro{\fingamefull}{\Gamma(\players,\pures,\pay)}	% for finite game, full description
\newmacro{\mixgame}{\Delta(\fingame)}	% for mixed extension

\newmacro{\minmax}{L}	% for minmax objective

\newmacro{\minvar}{\point_{1}}	% for generic min variable
\newmacro{\minvaralt}{\alt\minvar}	% for second min variable
\newmacro{\minvars}{\points_{1}}	% for set of minvars

\newmacro{\maxvar}{\point_{2}}	% for generic max variable
\newmacro{\maxvaralt}{\alt\maxvar}	% for second max variable
\newmacro{\maxvars}{\points_{2}}	% for set of maxvars

\newmacro{\pot}{f}	% for potential function

\newmacro{\hreg}{h}	% for regularizer
\newmacro{\breg}{D}	% for Bregman divergence
\newmacro{\mprox}{P}	% for prox-mapping
\newmacro{\mirror}{Q}	% for mirror map
\newmacro{\fench}{F}	% for Fenchel coupling
\newmacro{\hstr}{K}	% for strong convexity modulus
\newmacro{\hrange}{H}	% for regularizer depth
\newmacro{\proxdom}{\points_{\hreg}}	% for prox-domain

\DeclarePairedDelimiterXPP{\bregof}[2]{\breg}{(}{)}{}{#1,#2}	% for Bregman divergence
\DeclarePairedDelimiterXPP{\proxof}[2]{\mprox_{#1}}{(}{)}{}{#2}	% for Bregman prox step

\newmacro{\zone}{\mathbb{D}}	% for Bregman zone
\newop{\Eucl}{\Pi}	% for Euclidean projection
\newop{\logit}{LC}	% for logit choice
\newop{\dkl}{KL}	% for Kullback Leibler

\newmacro{\dpoint}{y}	% for generic dual point
\newmacro{\dpointalt}{\alt\dpoint}	% for second dual point
\newmacro{\dpointaltalt}{\altalt\dpoint}	% for third dual point
\newmacro{\dpoints}{\mathcal{Y}}	% for set of dual points
\newmacro{\dstate}{Y}	% for dual state

\newmacro{\flowmap}{\Theta}	% for (semi)flow map
\DeclarePairedDelimiterXPP{\flowof}[2]{\flowmap_{#1}}{(}{)}{}{#2}	% for flow

\newmacro{\traj}{x}	% for generic trajectory
\newmacro{\difftraj}{\dot\traj}	% for velocity of trajectory

\DeclarePairedDelimiterXPP{\trajof}[1]{\traj}{(}{)}{}{#1}	% for trajectory
\DeclarePairedDelimiterXPP{\difftrajof}[1]{\difftraj}{(}{)}{}{#1}	% for velocity of trajectory

\newmacro{\trajalt}{y}	% for second trajectory
\newmacro{\trajaltalt}{z}	% for third trajectory

\newmacro{\signal}{\hat\vecfield}	% for signal
\newmacro{\step}{\gamma}	% for step-size
\newmacro{\learn}{\eta}	% for learning rate

\newmacro{\runtime}{\tau}	% for proper time
\newmacro{\error}{Z}	% for error
\newmacro{\noise}{U}	% for noise
\newmacro{\bias}{b}	% for bias

\newmacro{\sbound}{M}	% for signal bound
\newmacro{\nbound}{\sdev}	% for noise bound
\newmacro{\bbound}{B}	% for bias bound

\newmacro{\snoise}{\xi}	% for scalar noise
\newmacro{\sbias}{\chi}	% for scalar bias

\newmacro{\mix}{\delta}	% for query radius
\newmacro{\perturb}{z}	% for perturbation
\newmacro{\pivot}{\point}	% for pivot point

\newmacro{\vertex}{v}	% for generic vertex
\newmacro{\vertexalt}{w}	% for second vertex
\newmacro{\vertexaltalt}{u}	% for third vertex
\newmacro{\nVertices}{V}	% for number of vertices
\newmacro{\vertices}{\mathcal{V}}	% for set of vertices

\newmacro{\edge}{e}	% for generic edge
\newmacro{\edgealt}{\alt\edge}	% for second edge
\newmacro{\edgealtalt}{\altalt\edge}	% for third edge
\newmacro{\nEdges}{E}	% for number of edges
\newmacro{\edges}{\mathcal{\nEdges}}	% for set of edges

\newmacro{\graph}{\mathcal{G}}	% for graph
\newmacro{\graphfull}{\graph(\vertices,\edges)}	% for graph, full description

\newmacro{\source}{O}	% for origin
\newmacro{\sink}{D}	% for destination

\newmacro{\pair}{i}	% for generic pair
\newmacro{\pairalt}{j}	% for second pair
\newmacro{\pairaltalt}{k}	% for third pair
\newmacro{\nPairs}{N}	% for number of pairs
\newmacro{\pairs}{\mathcal{\nPairs}}	% for set of pairs

\newmacro{\route}{p}	% for generic route
\newmacro{\routealt}{\alt\route}	% for second route
\newmacro{\routealtalt}{\altalt\route}	% for third route
\newmacro{\nRoutes}{P}	% for number of routes
\newmacro{\routes}{\mathcal{\nRoutes}}	% for set of routes

\newmacro{\flow}{f}	% for generic flow profile
\newmacro{\flowalt}{\alt\flow}	% for second flow
\newmacro{\flowaltalt}{\altalt\flow}	% for third flow
\newmacro{\flows}{\mathcal{F}}	% for set of flows

\newmacro{\load}{x}	% for generic load
\newmacro{\loadalt}{\alt\load}	% for second load
\newmacro{\loadaltalt}{\altalt\load}	% for third load
\newmacro{\loads}{\mathcal{X}}	% for set of loads

\newmacro{\meas}{\mu}	% for measure
\newmacro{\toler}{\eps}	% for tolerance

\newcommand{\tn}{\bar{\theta}^\lambda_n}
\newcommand{\E}{\mathbb{E}}

\newcommand{\hl}{h_\lambda}

\newcommand{\pt}{\hat{\pi}_t}
\newcommand{\pkl}{\hat{\pi}_{k\lambda}}

\newcommand{\ptk}{\hat{\pi}_{t|\mathcal{F}_{k\lambda}}}

\newcommand{\ur}{u_{\reg,\lambda}}
\newcommand{\hr}{\nabla u_{r,\lambda}}
\newcommand{\htl}{h_{r,\lambda}}
\newcommand{\pr}{{\pi_{\reg}}}

\newcommand{\lref}[1]{$\mathbf{B}$\ref{#1}}
\newcommand{\rd}{\rho^{\reg}_{n}}

\addauthor[Pan]{PM}{MediumBlue}

\newop{\poly}{poly}
\newop{\reg}{reg}
\newop{\LSI}{LSI}
\newop{\PI}{PI}
\newop{\ULA}{ULA}

\newmacro{\CPI}{\Const_{\PI}}
\newmacro{\CLSI}{\Const_{\LSI}}

\newmacro{\occmeas}{\mu}	% for measure
\newmacro{\size}{\delta}	% for size
\begin{document}

%**********************************************************************
%***    FRONT MATTER AND METADATA
%**********************************************************************

%----------------------------------------------------------------------
%% TITLE & AUTHORS
%----------------------------------------------------------------------
\title
{Tamed Langevin Sampling under Weaker Conditions}	% for title

%-------------------------------------------------------------------
\author
[I.~Lytras]
{Iosif Lytras$^{c,\ast,\diamond}$}
\address{$^{c}$\,%
Corresponding author.}
\address{$^{\ast}$\,%
Athena/Archimedes Research Centre, Athens, Greece.}
\address{$^{\diamond}$\,%
The University of Edinburgh, Edinburgh, UK.}
\EMAIL{ i.lytras@sms.ed.ac.uk}
%-------------------------------------------------------------------
\author
[P.~Mertikopoulos]
{Panayotis Mertikopoulos$^{\sharp}$}
\address{$^{\sharp}$\,%
Univ. Grenoble Alpes, CNRS, Inria, Grenoble INP, LIG, 38000 Grenoble, France.}
\EMAIL{panayotis.mertikopoulos@imag.fr}
%-------------------------------------------------------------------

%----------------------------------------------------------------------
%% KEYWORDS
%----------------------------------------------------------------------

%----------------------------------------------------------------------
%% ACKNOWLEDGMENTS
%----------------------------------------------------------------------
%\thanks{}

%----------------------------------------------------------------------
%% ACRONYMS
%----------------------------------------------------------------------
\newacro{LHS}{left-hand side}
\newacro{RHS}{right-hand side}
\newacro{iid}[i.i.d.]{independent and identically distributed}
\newacro{lsc}[l.s.c.]{lower semi-continuous}
\newacro{usc}[u.s.c.]{upper semi-continuous}
\newacro{rv}[r.v.]{random variable}
\newacro{wp1}[w.p.$1$]{with probability $1$}

\newacro{NE}{Nash equilibrium}
\newacroplural{NE}[NE]{Nash equilibria}

\newacro{GD}{gradient descent}
\newacro{KL}{Kullback\textendash Leibler}
\newacro{TV}{total variation}
\newacro{LSI}{logarithmic Sobolev inequality}
\newacro{PI}{Poincaré inequality}
\newacro{SGLD}{stochastic gradient Langevin dynamics}
\newacro{SDE}{stochastic differential equation}
\newacro{ULA}{unadjusted Langevin algorithm}
\newacro{WC}{weak convexity}

\newacro{wd}[wd-TULA]{weakly dissipative tamed unadjusted Langevin algorithm}
\newacro{reg}[reg-TULA]{regularized tamed unadjusted Langevin algorithm}

%----------------------------------------------------------------------
%% ABSTRACT
%----------------------------------------------------------------------
\begin{abstract}
%----------------------------------------------------------------------
%%% ABSTRACT
%----------------------------------------------------------------------
% !TEX root = ./Main.tex
%
%
Motivated by applications to deep learning which often fail standard Lipschitz smoothness requirements, we examine the problem of sampling from distributions that are not log-concave and are only weakly dissipative, with log-gradients allowed to grow superlinearly at infinity.
In terms of structure, we only assume that the target distribution satisfies either a Log-Sobolev  or a Poincare inequality and a local Lipschitz smoothness assumption with modulus growing possibly polynomially at infinity.
This set of assumptions greatly exceeds the operational limits of the ``vanilla'' ULA, making sampling from such distributions a highly involved affair.
To account for this, we introduce a taming scheme which is tailored to the growth and decay properties of the target distribution, and we provide explicit non-asymptotic guarantees for the proposed sampler in terms of the KL divergence, total variation, and Wasserstein distance to the target distribution.
\end{abstract}
\acresetall
\acused{iid}

%**********************************************************************
%***    BODY TEXT
%**********************************************************************
\allowdisplaybreaks	% for breaking long displays
\acresetall	% for resetting acros
\maketitle

%----------------------------------------------------------------------
%% INTRODUCTION
%----------------------------------------------------------------------
\section{Introduction}
\label{sec:introduction}
%----------------------------------------------------------------------
%%% INTRODUCTION
%----------------------------------------------------------------------
% !TEX root = ../Main.tex

A broad array of modern and emerging machine learning architectures relies on being able to sample efficiently from a target distribution $\pi$ on $\R^{d}$, typically expressed in Gibbs form as $\pi(x) \propto \exp(-u(x))$ for some potential function $u\from\R^{d}\to\R$.
Under suitable assumptions for $u$, this distribution arises naturally as the invariant measure of the Langevin \acl{SDE}
\begin{equation}
\label{eq:LSDE}
\tag{LSDE}
dX_{t}
	= -\nabla u(X_{t}) \dd t
	+ \sqrt{2} \dd B_{t}
\end{equation}
where $B_{t}$ is a canonical Wiener process (Brownian motion) in $\R^{d}$ with unit volatility.
Based on this key property of \eqref{eq:LSDE},
one of the most \textendash\ if not \emph{the} most \textendash\ widely used algorithmic schemes for sampling from $\pi$ is the so-called \acdef{ULA}, given in recursive form as
\begin{equation}
\label{eq:ULA}
\tag{ULA}
\theta_{n+1}^{\ULA}
	= \theta_{n}^{\ULA}
		- \lambda h\parens[\big]{\theta_{n}^{\ULA}}
		+ \sqrt{2\lambda} \xi_{n+1}
\end{equation}
where
$\theta_{n} \in \R^{d}$, $n=1,2,\dotsc$, is the algorithm's state variable,
$\xi_{n}$ is an \ac{iid} sequence of standard $d$-dimensional random variables with unit covariance,
$\lambda > 0$ is a step-size parameter,
and
$h \defeq \nabla u$ denotes the gradient of $u$.
The idea behind \eqref{eq:ULA} is that $\theta_{n}^{\ULA}$ can be seen as an Euler-Maruyama discretization of \eqref{eq:LSDE} so, for sufficiently large $n$ and small enough $\lambda$, $\theta_{n}^{\ULA}$ will be distributed according to some approximate version of the invariant measure of \eqref{eq:LSDE}, which is precisely the target distribution $\pi$.

This simple idea has generated a vast corpus of literature and techniques for proving the non-asymptotic convergence rate of \eqref{eq:ULA} in different probability metrics, the most popular ones being the Wasserstein and \acl{TV} distances, as well as the \acf{KL} and/or Rényi divergence.
Much of this literature has focused on the case where the target distribution $\pi$ is log-concave and has Lipschitz continuous log-gradients, corresponding respectively to convexity and Lipschitz smoothness of the potential $u$;
for some representative recent works, see \citet{dalalyan2017theoretical}, \citet{durmus2017nonasymptotic,durmus2019high}, \citet{convex} and references therein.

Beyond these works, especially when the target distribution is multimodal, there has been significant effort to relax the (strong) convexity requirement for $u$ by means of a combination of ``convexity at infinity'' and ``dissipativity'' assumptions \textendash\ that is, convexity outside a compact set, and a drift coercivity condition of the form $\langle{h(x)},{x}\rangle = \Omega(\abs{x}^{2})$ for the drift $h = \nabla u$ of \eqref{eq:LSDE} respectively, \cf \citet{berkeley}, \citet{majka2020nonasymptotic}, \citet{erdogdu2022convergence}, as well as a recent thread of results on the related \ac{SGLD} scheme by \citet{raginsky}, \citet{nonconvex} and \citet{zhang2023nonasymptotic}.

At the same time, building on an important insight of \citet{vempala2019rapid}, a parallel thread in the literature has explored at depth the role of isoperimetric inequalities in establishing the (rapid) convergence of \eqref{eq:ULA} when the potential of $\pi$ is Lipschitz smooth, either via the use of a \acdef{LSI} in the case of \citet{mou2022improved} and \citet{chewi2021analysis}, or a  \acdef{PI} in the case of \citet{balasubramanian2022towards}, and even weaker inequalities in \citet{mousavi2023towards} possibly reducing the degree of smoothness to (global) Hölder continuity of the drift of \eqref{eq:LSDE}, \cf \citet{nguyen2021unadjusted} and \citet{erdogdu2021convergence},\citet{mousavi2023towards}.
\smallskip

\para{Our contributions in the context of related work}

Our paper seeks to bridge these branches of the sampling literature \textendash\ the relaxation of global Lipschitz smoothness requirements and the relaxation of convexity requirements via the use of isoperimetric inequalities \textendash\ and, in so doing, to bring together the best of both worlds.
Specifically, motivated by applications to the optimization and sampling of deep learning models (which are notoriously non-Lipschitz), we seek to answer the following question:
\begin{center}
\itshape
How to sample efficiently in the absence of log-concavity and linear gradient growth properties?
Can one derive bounds in different distances with weaker assumptions?
\end{center}

This is a difficult setting for sampling because, as has been noted in several works, both \eqref{eq:ULA} and its \ac{SGLD} variants may be highly unstable in such scenarios;
in particular, when the drift coefficient of \eqref{eq:ULA} exhibits superlinear growth, the Euler-Marauyama scheme \textendash\ which forms the core component of \eqref{eq:ULA} \textendash\ diverges in a very strong sense.
A key result in this direction was obtained by \citet{hutzenthaler2011} who showed that the difference of the exact solution of a \ac{SDE} and its numerical approximation, diverges to infinity in the strong mean square sense, even at a finite point in time.
This negative result has shown that superlinear growth of the drift coefficient directly results in a blow-up of the moments of the numerical approximation scheme used to generate samples, which thus explains the failure of these algorithms.

Providing an efficient work-around to this issue is not easy, and one needs to explore the roots of \eqref{eq:ULA} for a possible answer \textendash\ 
specifically, going all the way back to the initial observation that \eqref{eq:ULA} is an Euler-Maruyama numerical apporximation scheme for the trajectories of \eqref{eq:LSDE}, and using the theory of numerical solutions of \acp{SDE} to explore a different angle of attack.
In this regard, a technology tailored to solving \acp{SDE} with superlinearly growing drifts first emerged in the works of \citet{hutzenthaler2012} and \citet{tamed-euler,SabanisAoAP}, revolving around a technique known as ``taming''.
The idea of these schemes is to create an adaptive Euler scheme with a new drift coefficient $\mu^{\lambda}$ which is a ``tamed'', rescaled version of the original drift, with the step-size of the algorithm appearing in the rescaling factor, and with the aim of ensuring the following dove-tailing properties:
\begin{enumerate}
[left=1em,label=\upshape(P\arabic*)]
%[label=\upshape(\itshape\alph*\hspace*{1pt}\upshape)]
\item
$\mu^{\lambda}$ has at most linear growth, that is, $\mu^{\lambda}(x) = \bigoh(\abs{x})$ for large $x$.
\item
$\mu^{\lambda}$ converges pointwise to $\mu$ in the limit $\lambda\to0$.
\end{enumerate}
This technique has been applied previously in the setting of Langevin-based sampling in multiple works under strong dissipativity or convexity assumptions \citep{tula,johnston2023kinetic}, in the stochastic gradient case \citep{TUSLA} under a ``convexity at infinity'' assumption \citep{neufeld2022non} and, in a concurrent work, under a \acl{LSI} coupled with a $2$-dissipativity assumption \citep{lytras2023taming}.
However, even though it is fairly common for distributions with superlinearly growing log-gradients to satisfy a \acl{LSI}, it is not always possible obtain a bound that remains well-behaved with respect to the dimension \textendash\ 
and similar limitations also hold for the $2$-dissipativity condition.

%----------------------------------------------------------------------
%% Assumptions table begins here

\begin{table}[tbp]
\footnotesize
\centering
%----------------------------------------------------------------------
%%% ASSUMPTIONS
%----------------------------------------------------------------------
% !TEX root = ../Main.tex

\begin{tabular}{lcc}
\toprule
	&\textsc{Convexity}
	&\textsc{Dissipativity}
	\\
\midrule
\citet{tula}
	&strongly convex (for $W_{2}$)
	&$\braket{x}{\nabla u(x)} \geq \const\abs{\nabla u(x)}\abs{x}$
	\\
\citet{neufeld2022non}
	&convex at infinity
	&$(2+r)$-dissipative
		\;
		\textpar{$r>0$}
	\\
\citet{lytras2023taming}
	&\acs{LSI}
	&$2$-dissipative
	\\
Current work
	&\acs{PI} + \acs{WC} / {LSI}
	&$1$-dissipative
%		\;
%		\textpar{$a \geq 1$}
	\\
\bottomrule
\end{tabular}
\caption{Comparison of convexity and dissipativity assumptions in related works.}
\label{tab:assumptions}
\end{table}

%% Assumptions table ends here
%----------------------------------------------------------------------

%----------------------------------------------------------------------
%% Rates table begins here

\begin{table}[tbp]
\footnotesize
\centering
%----------------------------------------------------------------------
%%% RATES
%----------------------------------------------------------------------
% !TEX root = ../Main.tex

\begin{tabular}{lcccc}
\toprule
	&\textsc{\acs{KL} Divergence}
	&\textsc{Total Variation}
	&\textsc{Wasserstein}
	&\textsc{Constants}
	\\
\midrule
\citet{tula}
	&\textemdash
	&$\tilde\bigoh(1/\eps^{2})$
	&$\tilde\bigoh(1/\eps)$ \hphantom{$^{2}$}\; [$W_{2}$]
	&$\exp(\bigoh(d))$
	\\
\citet{neufeld2022non}
	&\textemdash
	&\textemdash
	&$\tilde\bigoh(1/\eps^{2})$ \; [$W_{2}$]
	&$\exp(\bigoh(d))$
	\\
\citet{lytras2023taming}
	&$\tilde\bigoh(1/\eps)$
	&$\tilde\bigoh(1/\eps^{2})$
	&$\tilde\bigoh(1/\eps^{2})$ \; [$W_{2}$]
	&$\poly(d)/\CLSI$
	\\
  Current work(under \ac{LSI})
 &$\tilde\bigoh(1/\eps)$
	&$\tilde\bigoh(1/\eps^{2})$
	&$\tilde\bigoh(1/\eps^{2})$ \; [$W_{2}$]
	&$\poly(d)/\CLSI$
	\\
Current work (under \ac{PI})
	&$\tilde\bigoh(1/\eps^{3})$
	&$\tilde\bigoh(1/\eps^{6})$
	&$\tilde\bigoh(1/\eps^{8})$ \; [$W_{1}$]
	&$\poly(d)/\CPI$
	\\
\bottomrule
\end{tabular}
\caption{Comparison of convergence rates under different assumptions;
factors that are logarithmic in $1/\eps$ have been absorbed in the tilded $\tilde\bigoh(\argdot)$ notation.
The observed drop relative to the concurrent work of \citet{lytras2023taming} is due to the weaker assumptions made in our paper \textendash\ Poincaré \vs log-Sobolev and weak dissipativity \vs $2$-dissipativity (or higher), \cf \cref{tab:assumptions}.
The constants $\Const_{\LSI}$ and $\Const_{\PI}$ refer to the (positive) constants that appear in the log-Sobolev and Poincaré inequalities respectively.}
\vspace{-\baselineskip}
\label{tab:rates}
\end{table}

%% Rates table ends here
%----------------------------------------------------------------------

The main contribution of our paper is to provide a bridge between these two worlds and bring to the forefront the best properties of both:
sampling efficiently from potentials with locally Lipschitz log-gradients that may grow polynomially at infinity, with a drift coefficient that is only $1$-dissipative (instead of $2$-dissipative) and a considerably lighter \acl{PI} requirement \textendash\ as opposed to the more rigid framework imposed by the use of \acp{LSI}.
Specifically, we propose two novel algorithmic schemes,
the \acdef{wd}
and
the \acdef{reg}
which allow us to simultaneously treat superlinearly growing drift coefficients for target distributions satisfying a \acl{PI}, the former under a \ac{WC} requirement, the latter without.
For completeness, we also show that the proposed taming schemes achieve optimal convergence rates in the presence of stronger \ac{LSI} conditions.

To position these contributions in the context of related work, \cref{tab:assumptions} summarizes our paper's assumptions relative to the most closely related works in the literature, and \cref{tab:rates} provides a side-by-side comparison of the achieved rates.
To the best of our knowledge, the work closest to our own is the concurrent work of \citet{lytras2023taming}, who provide a tamed algorithmic scheme achieving an $\tilde\bigoh(1/\eps)$ rate of convergence to $\pi$ in the \ac{KL} divergence metric (respectively $\tilde\bigoh(1/\eps^{2})$ in terms of the total variation and Wasserstein $W_{2}$ distance).
Due to the relaxation from a \acl{LSI} to a considerably weaker \acl{PI}, our rates do not match those of \citet{lytras2023taming} in the case of the \ac{KL} and \ac{TV} metrics \textendash\ where \ac{wd} achieves a rate of $\tilde\bigoh(1/\eps^{3})$ and $\tilde\bigoh(1/\eps^{6})$ respectively, \cf \cref{tab:rates} (the Wassserstein metrics are otherwise incomparable; see also \cref{thm:reg-TULA} for the rates without any \ac{WC} requirements).
This also applies to the ``convexity at infinity'' assumption of \citet{neufeld2022non}, which has been shown to imply an \ac{LSI}, and is thus considerably more stringent than the \ac{PI} setting of our paper.
Importantly, our analysis still carries a polynomial dependence on the dimensionality of the problem, in contrast to the analysis of \citet{neufeld2022non} where the dependence is exponential.
These aspects of our results are particularly intriguing for future work on the subject, as they open a hitherto unexplored link between the isoperimetric inequalities, the role of coercivity in the target distribution in a superlinearly-growing gradient setting.

%----------------------------------------------------------------------
%% SETUP
%----------------------------------------------------------------------
\section{Setup and blanket assumptions}
\label{sec:setup}
%----------------------------------------------------------------------
%%% SETUP
%----------------------------------------------------------------------
% !TEX root = ../Main.tex

In this section, we provide the necessary groundwork for stating the proposed algorithmic schemes and our main results.

%----------------------------------------------------------------------
%%% Notation
%----------------------------------------------------------------------
\subsection{Notational conventions}

We begin by fixing notation and terminology.
Throughout our paper,
$\abs{\argdot}$ denotes the Euclidean norm of a vector,
$\norm{\argdot}$ the spectral norm of matrix;
the Frobenius norm will be denoted by $\frobnorm{\argdot}$,
and
the total variation distance by $\tvnorm{\argdot}$.
%The Euclidean norm of a vector $b \in \mathbb{R}^d$, the spectral norm and the Frobenius norm of a matrix $\sigma \in \mathbb{R}^{d \times m}$ are denoted by $|b|,||A||$ and $||A||_{\mathrm{F}}$ respectively.
%$A^{\top}$ is the transpose matrix of $A$.
For a sufficiently smooth function $f\from\R^{d}\to\R$, we will write $\nabla f$, $\nabla^{2}f$ and $\Delta f$ for its gradient, Hessian matrix, and Laplacian respectively,
and $J^{(i)} f$ for its $i$-th order Jacobian.
%Denote by $\nabla f, \nabla^2 f$ and $\Delta f$ the gradient of $f$, the Hessian of $f$ and the Laplacian of $f$ respectively.
%We denote the $i-th$ order Jacobian of an $i$ times differentiable function $f: \mathbb{R}^d \to \mathbb{R}$ as $J^{(i)}(f).$
We also write $\mathcal{H}^k$ for the usual Sobolev space.
%$\|\cdot\|_V$ is the total variation denoted by $\|\cdot\|_{T V}$.

For any two probability measures $\mu$, $\nu$ on a measurable space $\Omega$ with a $\sigma$-algebra understood from the context, we will write $d\mu/d\nu$ for the Radon-Nikodym derivative of $\mu$ with respect to $\nu$ when $\mu$ is absolutely continuous relative to $\nu$ ($\mu \ll \nu$).
In this case, the \acdef{KL} divergence of $\mu$ with respect to $\nu$ is defined as
\begin{equation}
\label{eq:KL}
\tag{KL}
{H}_\nu(\mu)=\int_{\Omega} \frac{d \mu}{d \nu} \log \left(\frac{d \mu}{d \nu}\right) d \nu.
\end{equation}
We say that $\zeta$ is a \define{transference plan} of $\mu$ and $\nu$ if it is a probability measure on $\R^{d}\times\R^{d}$ (endowed with the standard Borel algebra)
and we have $\zeta\left(A \times \mathbb{R}^d\right)=\mu(A)$ and $\zeta\left(\mathbb{R}^d \times A\right)=\nu(A)$ for every Borel subset $A$ of $\R^{d}$.
We denote by $\Pi(\mu, \nu)$ the set of transference plans of $\mu$ and $\nu$.
Furthermore, we say that a couple of $\mathbb{R}^d$-valued random variables $(X, Y)$ is a coupling of $\mu$ and $\nu$ if there exists $\zeta \in \Pi(\mu, \nu)$ such that $(X, Y)$ is distributed according to $\zeta$.
Finally, for two probability measures $\mu$ and $\nu$ on $\R^{d}$, the Wasserstein distance of order $p \geq 1$ is defined as
\begin{equation}
W_p(\mu, \nu)
	=\left(\inf _{\zeta \in \Pi(\mu, \nu)} \int_{\R^d \times \R^d}\abs{x-y}^p \dd\zeta(x, y)\right)^{1 / p}.
\end{equation}

%Let $\mu$ and $\nu$ be two probability measures on a state space $\Omega$ with a given $\sigma$-algebra.
%If $\mu \ll \nu$, we denote by $d \mu / d \nu$ the Radon-Nikodym derivative of $\mu$ w.r.t.
%$\nu$.
%Then, the Kullback-Leibler divergence of $\mu$ w.r.t.
%$\nu$ is given by
%$$
%{H}_\nu(\mu)=\int_{\Omega} \frac{d \mu}{d \nu} \log \left(\frac{d \mu}{d \nu}\right) d \nu .
%$$
%We say that $\zeta$ is a transference plan of $\mu$ and $\nu$ if it is a probability measure on $\left(\mathbb{R}^d \times \mathbb{R}^d, \mathcal{B}\left(\mathbb{R}^d\right) \times \mathcal{B}\left(\mathbb{R}^d\right)\right)$ such that for any Borel set $A$ of $\mathbb{R}^d, \zeta\left(A \times \mathbb{R}^d\right)=\mu(A)$
%and $\zeta\left(\mathbb{R}^d \times A\right)=\nu(A)$.
%We denote by $\Pi(\mu, \nu)$ the set of transference plans of $\mu$ and $\nu$.
%Furthermore, we say that a couple of $\mathbb{R}^d$-valued random variables $(X, Y)$ is a coupling of $\mu$ and $\nu$ if there exists $\zeta \in \Pi(\mu, \nu)$ such that $(X, Y)$ is distributed according to $\zeta$.
%For two probability measures $\mu$ and $\nu$, the Wasserstein distance of order $p \geq 1$ is defined as
%$$
%W_p(\mu, \nu)=\left(\inf _{\zeta \in \Pi(\mu, \nu)} \int_{\mathbb{R}^d \times \mathbb{R}^d}\abs{x-y}^p d \zeta(x, y)\right)^{1 / p}.
%$$

%----------------------------------------------------------------------
%%% Blanket assumptions
%----------------------------------------------------------------------
\subsection{Blanket assumptions}

Throughout what follows, we will write
$\pi \defeq e^{-u} \big/ \int e^{-u}$ for the target distribution to be sampled,
and
$Lf = \Delta f - \Gamma(u,f)$ for the infinitesimal generator of \eqref{eq:LSDE}, where $\Gamma(f,g) = \braket{\nabla f}{\nabla g}$ denotes the carré du champ operator for $f,g \in \mathcal{H}^{1}$.

We begin by stating our blanket assumptions for \eqref{eq:LSDE}:

\begin{assumption}
\label{asm:drift}
The drift $h = \nabla u$ of \eqref{eq:LSDE} satisfies the following conditions:
\begin{enumerate}
[left=\parindent,label=\upshape(A\arabic*)]
\item
\label[assumption]{asm:PLC}
\label[assumption]{ass-pol lip}
\define{Polynomial Lipschitz continuity:}
\begin{equation}
\label{eq:Lips}
\tag{PLC}
\abs{h(x) - h(y)}
	\leq L'(1 + \abs{x} + \abs{y})^{l'} \abs{x-y}
%	\quad
%	\text{for some $L',l'>0$ and for all $x,y\in\R^{d}$.}
\end{equation}
for some $L',l'>0$ and for all $x,y\in\R^{d}$
\item
\label[assumption]{asm:wd}
\label[assumption]{ass-2dissip}
\define{Weak dissipativity:}
\begin{equation}
\label{eq:wd}
\tag{WD}
\braket{h(x)}{x}
	\geq A \abs{x}^{a} - b
\end{equation}
for some $a\geq1$, $A,b>0$ and for all $x\in\R^{d}$.
\item
\label[assumption]{asm:PJG}
\label[assumption]{ass-derivbound}
\define{Polynomial Jacobian growth:}
\begin{equation}
\label{eq:PJG}
\tag{PJG}
\max\{\abs{h(x)},\norm{J^{(i)}(h)(x)}\}\leq L(1+\abs{x}^{2l})
%	\quad \forall x\in \mathbb{R}^d
\end{equation}
for some $L,l>0$ and for all $x\in\R^{d}$.
\end{enumerate}
\end{assumption}

Of the above, \cref{asm:PLC} posits that $h$ is locally Lipschitz continuous, with the modulus of Lipschitz continuity (essentially the largest eigenvalue of the Hessian of $u$) growing possibly at a polynomial rate at infinity.
As such, \cref{asm:PLC} allows us to capture a very broad spectrum of applications with superlinear Hessian growth (especially in the context of deep learning landscapes that exhibit polynomial growth with a degree equal to the depth of the underlying network).

\Cref{asm:wd} is at the core of our analysis, as it enables us to provide moment bounds that are uniform in time:
it is essentially a coercivity assumption, but with a relaxed exponent relative to the $2$-dissipativity framework of other works, which can be fairly restrictive if the tails of the target distribution are thicker than sub-Gaussians.
For generality, we treat not only the case of $1$-dissipative gradients, but all dissipativity exponents $a\geq1$.
In practice, although it is quite possible that $2$- dissipativity may hold for a given distribution, $a$-dissipativity for $a<2$ may be much easier to verify and leverage to produce more favourable constants $A$,$b$. 

Finally, \cref{asm:PJG} is a strictly technnical requirement intended to streamline our presentation when rigorously differentiating under the integral sign \textendash\ that is, exchanging the order of integration and time derivatives \textendash\ in the use of a divergence theorem when establishing a differential inequality later in our paper.

Our next blanket assumption concerns the target distribution $\pi$:

\begin{assumption}
\label{asm:target}
The target distribution $\pi$ satisfies a \acdef{PI} of the form
\begin{equation}
\label{eq:PI}
%\tag{$\PI(\varrho)$}
\tag{PI}
\operatorname{Var}_{\pi}(f)
	\defeq \int_{\mathbb{R}^d}\left(f-\int_{\mathbb{R}^d} f \dd\pi\right)^2 \dd\pi
		\leq \frac{1}{\CPI} \int\abs{\nabla f}^2 \dd\pi
\end{equation}
for some positive constant $\CPI>0$ and all test functions $f \in \mathcal{H}^1(\R^d)$.
\end{assumption}

This assumption is weaker than the widely used (but more stringent) \acl{LSI}. 
\begin{equation}
\label{eq:LSI}
%\tag{$\LSI(\alpha)$}
\tag{LSI}
H_{ \pi}(\nu)
	\defeq \int_{\mathbb{R}^d} f\log f \dd\pi
	\leq \frac{1}{2\CLSI} \int_{\mathbb{R}^d} \frac{\Gamma(f,f)}{f} \dd\pi
	\eqdef \frac{1}{2\CLSI} I_{ \pi}(\nu),
\end{equation}
for some positive constant $\CLSI>0$ and for every probability measure $\nu \ll \pi$ with $f \defeq d\nu/d\pi$.
One should note that PI can be given as a consequence of the weak dissipativity condition (see \cite{bakry2008simple}).
%A Gibbs probability measure $ \pi$ satisfies the logarithmic Sobolev inequality with constant $\alpha>0$, denoted $\operatorname{LSI}(\alpha)$, if for all probability measures $\nu$ such that $\nu \ll  \pi$ $f:=\frac{d \nu}{d  \pi} \in H^1(\mathbb{R}^d)$
In the case of diffusion processes, the importance of these inequalities lies in the fact that \eqref{eq:PI} implies exponential ergodicity with respect to the $\chi^{2}$ divergence, while \eqref{eq:LSI} implies exponential ergodicity in relative entropy.
In particular, by the Bakry\textendash Emery theorem \citep{bakry2006diffusions},
\eqref{eq:LSI} was established for strongly convex potentials and is stable under bounded pertrubations, Lipschitz mappings and convolutions.
It also implies Talagrand's transporation-cost inequality:
if $\mu$ satisfies \eqref{eq:LSI}, then
\begin{equation}
W_2(\mu,\nu)\leq \sqrt{\frac{2}{\CLSI} H_\mu(\nu)}.
\end{equation}

By comparison, \eqref{eq:PI} is significantly less stringent than \eqref{eq:LSI}:
to begin, \eqref{eq:LSI} implies \eqref{eq:PI} with the same constant but, moreover, \eqref{eq:PI} has been shown to hold for dissipative potentials where \eqref{eq:LSI} fails, and is also stable under pertubations, Lipschitz mappings and convolutions.
It also implies exponential moments of some order \ie $\E_\mu e^{q |x|}<\infty$ whenever $\mu$ satisfies \eqref{eq:PI}.
%\[\mu \quad\text{satisfies PI} \implies \quad \exists q>0: \E_\mu e^{q |x|}<\infty.\]

Our last assumption concerns the convexity characteristics of the potential $u(x)$ and will be used only in the case where \eqref{eq:PI} is the strongest condition in place.
\begin{assumption}
\label{ass11}
\label{asm:WC}
$u$ is weakly convex, \ie
\begin{equation}
\label{eq:WC}
\tag{WC}
\nabla^{2}u(x)
	\succcurlyeq - K I
	\quad
	\text{for some $K>0$ and for all $x\in\R^{d}$}.
\end{equation}
\end{assumption}
%\begin{assC}\label{ass11}
%There exists $K>0$ such that \[\nabla^2 u(x)\geq -KI_d \quad \forall x\in \mathbb{R}^d.\] where $I_d$ is the $d\times d$ identity matrix.
%\end{assC}
This assumption simply means that the eigenvalues of $\nabla^{2}u(x)$ do not become arbitrarily negative, and is widely used in the numerical approximation of \acp{SDE}.
For our purposes, we will mostly use it in the context of the famous HWI inequality

\begin{equation}
\label{eq:HWI}
\tag{HWI}
H_\pi(\nu)
	\leq \sqrt{I_\pi(\nu)} {W_2(\pi,\nu)} +\frac{K}{2} W_2^2(\pi,\nu)
	\quad
	\text{for all $\pi,\nu \in \mathcal{P}_2(\R^d)$}.
\end{equation}
%Throughout our article we are going treat the following cases:

To provide a glimpse of the analysis to come, we will develop and examine two novel taming schemes, one non-regularized and one regularized, that are tailored to the dissipativity profile of the initial potential, and which cover the following cases:
\begin{enumerate}
\item
When the potential satisfies \eqref{eq:LSI}.
\item
For the non-regularized case:
when \eqref{eq:LSI} fails and the potential satisfies \eqref{eq:PI} along with \eqref{eq:WC}.
\item
For the regularized case:
when \eqref{eq:LSI} fails and the potential only satisfies \eqref{eq:PI}.
\end{enumerate}
%We create a novel taming scheme which captures the dissipativity of the initial potential to treat the following cases:
%\begin{enumerate}
%\item
%The potential satisfies a Log-Sobolev inequality.
%\item
%The potential satisfies a condition weaker than \ac{LSI}:  \acl{PI} along a weak convexity assumption.
%\end{enumerate}
%In addition, we provide another scheme using regularization where the weak convexity assumption is dropped and the potential only satisfies a Poincare inequality.

\subsection{The failure of the \acl{ULA}}

Before moving forward with the development of the taming schemes mentioned above, we conclude this section with a simple \textendash\ but not simplistic \textendash\ $1$-dimensional example that satisfies our range of assumptions, but where the ``vanilla'' \acl{ULA} fails.

Setting $u(x)=x^3/3$ and applying \eqref{eq:ULA} with step-size $\lambda$ and initial condition $X_0=\mathcal{N}(0,\frac{4}{\lambda})$, we get
\[X_{n+1}=X_n -\lambda X_n^2 +\sqrt{2\lambda}\xi_{n+1}\]
Then, since $X_n$ is independent of ${\xi}_{n+1}$
\[\E [X_{n+1}^2]= \E \left[ X_n^2 (1-\lambda X_n)^2\right] + 2\lambda= \E X_n^2 (1-2\lambda \E X_n + \lambda^2 \E X_n^2) +2\lambda \]
using the inequality $1-2x+x^2\geq -1 +\frac{1}{2}x^2$ one obtains
\[\E [X_{n+1}^2]\geq -\E X_n^2 +\frac{1}{2}\lambda^2 \E X_n^6\geq -\E X_n^2 +\frac{\lambda^2}{2} \left(\E X_n^2\right)^3 + 2\lambda\] where the last step was obtain by Jensen's inequality.
Applying for $n=0$ it is easy to see that
\[\E X_1^2 \geq (\E X_0^2)\left( \frac{\lambda^2}{2} \E (X_0^2)^2 -1\right) +2\lambda \geq \E X_0^2 +2\lambda.\]
Iterating over $n$ yields
\begin{equation}
\E X_{n+1}^2 \geq \E X_n^2 +2\lambda n.
\end{equation}
We thus see that the second moment of the algorithm's iterates diverges as $n\to\infty$, indicating in this way that \eqref{eq:ULA} cannot be used to sample from the target distribution.

%----------------------------------------------------------------------
%% RESULTS
%----------------------------------------------------------------------
\section{Tamed schemes and main results}
\label{sec:results}
%----------------------------------------------------------------------
%%% RESULTS
%----------------------------------------------------------------------
% !TEX root = ../Main.tex

We now proceed to state our tamed algorithmic schemes and main results.

%----------------------------------------------------------------------
%%% WC
%----------------------------------------------------------------------
\subsection{Taming without regularization}

To streamline our presentation and ease the introduction of the various components of the analysis, we begin with the case where the eigenvalues of $\nabla^{2}u$ do not become arbitrarily negative, \ie $u$ satisfies the weak convexity assumption \eqref{eq:WC}.
In this case, we will consider the \acli{wd} that iterates as
\begin{equation}
\label{eq-wdTULA}
\tag{wd-TULA}
\bar\theta_{n+1}^{\lambda}
	= \bar\theta_{n}^{\lambda}
		- \lambda h_{\lambda}(\bar\theta_{n}^{\lambda})
		+ \sqrt{2\lambda} \xi_{n+1}
	\quad
	\text{for all $n=0,1,\dotsc$}
\end{equation}
where
$\theta_{0}$ is initialized randomly according to a Gaussian distribution $\pi_{0}$,%
\footnote{The Gaussian requirement could be relaxed by positing that $\abs{\nabla\log\pi_{0}}$ and $\norm{\nabla^{2}\log\pi_{0}}$ grow at most polynomially, but we will not need this level of generality.}
$\xi_{n}$ is an \acs{iid} sequence of Gaussian $d$-dimensional vectors with unit covariance,
and the \emph{tamed drift} $h_{\lambda}$ is given by
\begin{equation}
\label{eq:pot-tamed}
h_{\lambda}(x)
	= \frac{Ax}{(1+\abs{x}^{2})^{1 - a/2}}
	+ f_{\lambda}(x)
	\quad
	\text{with}
	\quad
f_{\lambda}(x)
	= \frac{f(x)}{1 + \sqrt{\lambda}\abs{x}^{2\ell}}
\end{equation}
where $f(x)=h(x)-\frac{Ax}{(1+\abs{x}^{2})^{1 - a/2}}$
and the various constants defined as in \cref{asm:drift}.

To connect \eqref{eq-wdTULA} with the existing literature on tamed schemes, we note here that the majority of taming factors are either of the form $h(x) / [1+(\lambda)^c \abs{h(x)}]$ for $c=1$ or $c=1/2$ \citep{tula},
or of the form $h(x) / [1+\lambda^c |x|^{2l-1}]$ \citep{TUSLA}.
In this regard, the taming scheme \eqref{eq:pot-tamed} is more intricate:
we first split the original gradient drift into a part which has at most linear growth, and we then proceed to tame the superlinearly growing part.
The drift coefficient of this scheme has the property that it grows at most as $|x|^\frac{a}{2}$ (so it grows at most linearly)  while inheriting the dissipativity condition of the initial gradient. For more details, see Lemma \ref{alg-diss}.
In this regard, when the potential satisfies a stronger $2$-dissipativity condition, we recover the taming scheme of \citet{lytras2023taming}.

Our main result for \eqref{eq-wdTULA} may then be stated as follows:

\begin{theorem}
\label{thm:wd-TULA}
Suppose that \cref{asm:drift,asm:target,asm:WC} hold and let $\rho_{n}$ denote the distribution of the $n$-th iterate of \eqref{eq-wdTULA} run with $\lambda < \lambda_{\max}=\min\{\frac{1}{4(2AC^* +2L+1)^2},\frac{1}{\dot{c}_0 H_\pi(\rho_0)},\frac{2}{\mu^2}\}$ where the constants are given in the proof of Proposition \ref{eq-inequality h-I} and Lemmas \ref{expmom1}, Proposition \ref{theo-solvdiff} .

Then $\rho_{n}$ enjoys the convergence guarantee
\begin{equation}
H_\pi(\rho_n)
	\leq \parens*{1-\frac{\dot{c}_0}{2}\lambda^{3/2}}^n H_\pi(\rho_0)
		+ \parens*{1+\frac{4c_1}{c_0}} \sqrt{\lambda}\end{equation}
where $c_{1}$ depends polynomially on $d$ and $\dot c_{0}$ is an explicit function of the Poincaré constant $\CPI$ of \eqref{eq:PI}.
In particular, given a tolerance level $\eps>0$,
\eqref{eq-wdTULA} achieves $H_{\pi}(\rho_{n}) \leq \eps$
within
$n \geq \dot c_{0}^{-1} (1 + c_{1}/\dot c_{0})^{3} \log(2/\eps)/\eps^{3} = \tilde\Theta(1/\eps^{3})$
if run with step-size
$\lambda \leq \eps^{2} / [4 (1+4 C_1/\dot c_0)^2]$.
\end{theorem}

This theorem ensures that the algorithm converges at a polynomial rate, even in the absence of \eqref{eq:LSI}. This is very important in practice as, even if \eqref{eq:LSI} holds, it is usually difficult to derive explicit bounds with nice dependence on the problem's defining parameters.
More to the point, if \eqref{eq:LSI} holds and $\CLSI$ is known, the proposed algorithm exhibits optimal convergence rates, achieving in this way the best of both worlds:

\begin{theorem}
    \label{thm:wd-TULA-LSI}
    Suppose that \cref{asm:drift} and \eqref{eq:LSI} hold.
    Let $\rho_n$ be the distribution of $n-th$ iterate of the algorithm \eqref{eq-wdTULA}.
Then,  for $\lambda\leq \lambda_{\max}$, we have
\[H_{\pi}(\rho_n)\leq e^{-\frac{3}{2} C_{LSI}\lambda (n-1) } H_{\pi}(\rho_0) + \frac{ \hat{C}}{\frac{3}{2} C_{LSI}}\lambda\]
where $\hat{C}$ depends polynomially on the dimension with leading term at most $\mathcal{O}\left(d^{\max\{2,l'+1\}(2l+1)}\right)$.
In particular, given a tolerance level $\eps>0$, if \eqref{eq-wdTULA} is run with step-size $\lambda\leq \frac{3 \epsilon C_{LSI}}{2  \hat{C}}$, for $n\geq \frac{2  \hat{C}}{\epsilon}C_{LSI}^{-1}\log(\frac{2}{\epsilon}H_{\pi}(\rho_0))= \tilde\Theta(1/\eps) $ iterations, we have $H_\pi(\rho_n)\leq \epsilon.$ 
\end{theorem}
\Cref{thm:wd-TULA,thm:wd-TULA-LSI} are our main results for \eqref{eq-wdTULA}.
The proof of both theorems is fairly arduous and involves a series of intricate steps, so, to streamline our presentation and facilitate our comparison with the case where \eqref{eq:WC} is dropped altogether, we proceed directly to the regularized version of \eqref{eq-wdTULA} and defer the discussion of the proof of the theorem to the next section.

%----------------------------------------------------------------------
%%% Regularized taming
%----------------------------------------------------------------------
\subsection{Regularized taming}

We consider now the case where the eigenvalues of $\nabla^{2}u(x)$ become arbitrarily negative (\ie \cref{asm:WC} fails altogether).
To account for this negative growth, we are going to regularize the tamed potential by anchoring it close to the original target.
This will require care to ensure that the new sampling potential satisfies \cref{asm:WC} with a controllable constant, as well as the corresponding regularity requirements of \cref{asm:drift}.

Without further ado, these considerations lead to the \emph{regularized potential}
\begin{equation}
\label{eq:pot-reg}
u_{r,\lambda}(x)
	= u(x) + \lambda \abs{x}^{2r+2}
\end{equation}
where, with a fair degree of hindsight, the exponent $r$ is chosen so that $r>l/2$ and $r(2+l') / [(r+1)(2r-l')] < 1$ (a moment's reflection shows that this is not the empty set).

In view of the above, the \emph{regularized taming} scheme that we will consider involves rescaling by the factor $(1 + \sqrt{\lambda}\abs{x}^{2r+1})$, leading to the regularized drift:
\begin{equation}
h_{r,\lambda}(x)
	= \frac{A x }{(1+|x|^2)^{1-\frac{a}{2}}}
		+ \frac{\hr(x) - A x (1+|x|^2)^{a/2-1}}{1+\sqrt{\lambda}|x|^{2r+1}}.
\end{equation}
In this way, we obtain the \acli{reg}
\begin{equation}
\label{eq-tamedregalg}
\tag{reg-TULA}
\bar x_{n+1}^{\lambda}
	= \bar x_{n}^{\lambda}
		- \lambda \htl(\bar x_{n}^{\lambda})
		+ \sqrt{2\lambda} \Xi_{n+1}
\end{equation}
where $\Xi_n$ is an \acs{iid} sequence of Gaussian $d$-dimensional vectors.
Our main result for this regularized sampling scheme may then be stated as follows:

\begin{theorem}
\label{thm:reg-TULA}
Suppose that \cref{asm:drift,asm:target} hold and let $\rd$ denote the distribution of the $n$-th iterate of \eqref{eq-tamedregalg} run with $\lambda < \lambda_{\max,2}:=\min\{\lambda_{\max},\frac{ln2}{R_2^{2r+2}}\}$ where $R_2$ is given in Lemma \ref{lemma poincarereg}.
Then $\rd$ enjoys the convergence guarantee
\begin{equation}
H_\pi(\rd)
	\leq \parens*{1- c\lambda^{1+\frac{1}{r+1}+\frac{l}{2r-l}}}^{n-1} H_{\pr}(\rho_0)
		+ (\hat{C}/c) \lambda^{1-\frac{1}{r+1}-\frac{l}{2r-l}}
		+ C_3 \lambda
\end{equation}
where $c,C_3,\hat{C}$ depend polynomially on $d$.
In particular, if $c_{l,r}:= \frac{r(2+l)}{(r+1)(2r-l)}$, then, for $\lambda<\mathcal{O}\left({\epsilon}^\frac{1}{1-c_{l,r}}\right),$ we have
\[H_{\pr}(\rd)\leq \epsilon \quad \text{after} \quad  n = \Theta\left(\log(1/\epsilon) \cdot \epsilon^{-\frac{1+c_{l,r}}{1-c_{l,r}}}\right)\quad \text{iterations}.\]
%The reason that we pick this new regularized potential is the fact that it is $\lambda-$ close to the original and at the same it is ''strongly convex at infinity''.
%This allows us to retrieve the assumption \bref{ass11}.
\end{theorem}

\section{Proof outline and technical innovations}
\label{sec:proofs}
%----------------------------------------------------------------------
%%% PROOF OUTLINE
%----------------------------------------------------------------------
% !TEX root = ../Main.tex

To give an idea of the main ideas and technical innovations required for the proof of \cref{thm:wd-TULA,thm:wd-TULA-LSI,thm:reg-TULA}, we provide below a brief roadmap of our proof strategy.

The cornerstone of our approach is the derivation of a differential inequality in the spirit of \citet{vempala2019rapid}.
The tricky part here is that the drift coefficient is not the original gradient but a tamed one, which yields additional complexity when one tries to prove the exchange of integrals and derivatives starting from the Fokker-Planck equation.
In so doing, we ultimately obtain a ``template inequality'' of the form
\begin{equation}
\label{eq:template}
\frac{d}{dt}H_{\pi}(\pt)\leq-\frac{3}{4} I_{\pi} (\pt) +   \E | h(\theta_t)-h_{\lambda}(\theta_{k\lambda})|^2
\end{equation}
where $\pt$ is the continuous-time interpolation of the algorithm.

The first term of the template inequality \eqref{eq:template} is connected to the relative entropy via an isoperimtric inequality, either directly \textendash\ if \eqref{eq:LSI} holds \textendash\ or by means of another inequality \textendash\ under \eqref{eq:PI} and \eqref{eq:WC}.
In this last case, the connection is achieved by establishing a ``modified'' version of \eqref{eq:LSI} as a consequence of \eqref{eq:PI} and \eqref{eq:HWI}.
As for the second term of \eqref{eq:template}, its contribution can be controlled by the one-step error of the algorithm (by local Lipschitzness) and the $L_2$ approximation of the tamed scheme to the original gradient at grid points.
A key difficulty here is that these bounds must be uniform in the number of iterations moment bounds for our algorithm which is obtained by the careful construction of the tamed coefficient, which satisfies a dissipativity condition.

This is an important novetlyt of our work, as we are able to achieve a uniform-in-time exponential moment bound for the algorithm, even in the $1-$ dissipative case (contrary to other works such as \cite{erdogdu2021convergence} where the moments bounds are not uniform in time or \cite{zhang2023nonasymptotic} which leverages a ``convexity at infinity'' assumption).
We achieve this by means of a Herbst argument for Gaussians (since each iterate is a Gaussian when conditioned to the previous step), to pass from the conditional expectation of the exponential to the exponential of the conditional expectation. We then use the contraction structure (which is provided by the inherited dissipativity of our scheme and the growth of the tamed drift coefficient as given in Lemma \ref{alg-diss}) to create an induction.

In the absence of \emph{both} \eqref{eq:LSI} \emph{and} \eqref{eq:WC}, we employ a similar method to sample from the regularized potential \textendash\ for which we prove the equivalent of \eqref{eq:LSI} \textendash\  and we then proceed to compute the \ac{KL} divergence of the algorithm's iterates relative to the target distribution.
The main challenge here is to show that the regularized potential also satisfies \eqref{eq:PI} with a constant that is explicitly connected to the Poincaré constant of the \emph{original} target and is \emph{independent} of $\lambda$.
In general, these are mutually antagonistic properties, which are ultimately achieved in our case by leveraging the dissipativity properties of the regularized potential to find a Lyapunov function $W$ such that $LW\leq -\theta$ for some $\theta>0$ outside of a ball.
By using the inequality for the infinitesimal generator of the Langevin SDE with drift coeffient the regularized potential, and the fact that, inside said ball, the regularized measure inherits the Poincaré constant of the original (as a bounded perturbation thereof), and by employing a variation of a shrewd argument of \citet[Proof of Theorem 2.3]{cattiaux2013poincare}, we are finally able to establish \eqref{eq:PI} on the whole space.
The conditions of $2$-dissipativity and weak convexity are then easier to prove, eventually leading to the upgrade of \eqref{eq:PI} to a suitably modified form of \eqref{eq:LSI}.

%In the above, determining the explicit constants while simultaneously trying to minimize their dependence on the step-size as much as possible, has been quite an obstacle to overcome.
A major obstacle in the above strategy is determining the explicit constants while simultaneously trying to minimize their dependence on the step-size as much as possible.
For example, one can establish a version of \eqref{eq:LSI} for the regularized potential in a more direct manner, by simply using convexity at infinity, or by a technique similar to \citet[Corollary 5.4]{lytras2023taming}.
However, the version of \eqref{eq:LSI} obtained in this would involve a catastrophic exponential dependence on $1/\lambda$, which would thus render it unusable for deriving finite-time convergence rates.
%that would be exponential on the inverse of the stepsize.
Albeit (significantly) more involved, our method completely circumvents the exponential dependence, which in turn enables the polynomial-time convergence rates of the proposed schemes (and recoups the technical investment described above).

For convenience, we summarize the main steps below, in decreasing order of the assumptions made.
%We are going to propose two algorithms based on the same tamed scheme, for three particular cases:

\para{Case 1: Analysis under \eqref{eq:LSI}}
This case concerns \cref{thm:wd-TULA-LSI}, and the analysis unfolds as follows:
\begin{enumerate}
\item
We prove that the tamed coefficients exhibit linear growth, and inherit the dissipativity property of the original gradient, \cf \cref{alg-diss}.
\item
We use the properties of the tamed scheme to derive exponential (and subsequently polynomial) moments for our algorithm, \cf \cref{expmom1,lemma-algmom}. 
\item
We establish a differential template inequality for the relative entropy between the continuous-time interpolation of the algorithm and target measure, \cf \cref{diff-ineq}.
This involves a rigorous treatment of the exchange of derivatives and integrals.
\item
We bound the remaining terms using the local Lipschitz property of $\nabla u$ and the approximation properties of the tamed scheme.
\item
We employ \eqref{eq:LSI} to conntect the Fisher information term to the relative entropy, and we backsolve to produce convergence rates in terms of the \ac{KL} divergence.
\item
Using Pinsker's inequality, we obtain a result for total variation distance and by Talagrand's inequality for the $W_2$ distance.
\end{enumerate}

\para{Case 2: \eqref{eq:PI} and \eqref{eq:WC}}
This case concerns \cref{thm:wd-TULA}, and the analysis unfolds as follows:
\begin{enumerate}
\item
We employ the same scheme to tame $\nabla u$ and repeat the first steps as in the proof of \cref{thm:wd-TULA-LSI}.
\item
Lacking \eqref{eq:LSI}, our analysis branches out as follows:
we use \cref{asm:WC} and \eqref{eq:PI} to produce a different template inequality between the \ac{KL} divergence and the Fisher information distance between the continuous-time interpolation of the algorithm and target measure (\cref{eq-inequality h-I}).
\item
We backsolve the derived differential inequality to produce non-exponential convergence rates relative to the \ac{KL} divergence, \cf \cref{theo-solvdiff,theoKL1}.
\item
Finally, by using Pinsker's inequality and the exponential moments of our alogirthm and the bound in relative entropy, we are able to derive bounds $TV$ and $W_1$ distance under \eqref{eq:PI}, \cf \cref{cor-othdist}. 
\end{enumerate}

\para{Case 3: \eqref{eq:PI} only}
This case concerns \cref{thm:reg-TULA}, and the analysis unfolds as follows:
%Here, the relevant result is \cref{thm:reg-TULA} for the regularized scheme \eqref{eq-tamedregalg}, and the main steps are as follows:
\begin{enumerate}
\item
We introduce a regularized potential to sample from, and we show that it has a range of desirable properties as described in \cref{alg-diss}.
\item
We use the same scheme to tame the gradient of the regularized potential and, through similar arguments, we derive exponential (and polynomial) moments for \eqref{eq-tamedregalg}.
\item
We show that the regularized measure satisfies a version of \eqref{eq:PI}, \cf \cref{lemma poincarereg}, and this inequality can be upgraded to a version of \eqref{eq:LSI} with manageable constants (\cref{LSI reg}).
\item
We branch back to the analysis using \eqref{eq:LSI} to sample from the regularized potential, and we use the relation between the regularized and the original one to derive the algorithm's convergence rate, as outlined in \cref{theo2}.
\item
Using Pinsker's and Talagrand's inequalities, we convert these bounds to $TV$ and $W_2$ (\cref{cor-22}).
\end{enumerate}

The details of all the above are provided in full in the paper's appendix.

\section{Numerical Experiments}
\label{sec:numerics}
%----------------------------------------------------------------------
%%% NUMERICS
%----------------------------------------------------------------------
% !TEX root = ../Main.tex

We proceed by providing some numerical experiments that validate our results.
The focus of our attention will be the invariant measure $\pi$ generated by the double-well potential $u=(|x|^2-1)^2$.
Since $\nabla u=4x(|x|^2-1)$ one easily observes that the potential satisfies our smoothness and dissipativity assumptions.
Since the 1-dissipativity holds one can also deduce a Poincare inequality for the potential.
Finally, it is immediate to see that the potential satisfies a convexity at infinity assumption, which of course is way stronger that our weak convexity assumption.
For explicit calculations the interested reader can point to \cite{neufeld2022non}.
We perform the experiment for our algorithm for $10^6$ iterations and $10^5$ samples for 100 independent iterations. The dimension is $d=100$ and we start the algorithm from a constant where every coordinate is zero, but the first coordinate is 200. We should note that starting from a constant doesn't contradict our analysis as we can still perform the analysis of the algorithm starting from the result of the 1st iteration which is a Gaussian.
When one runs 'vanilla' ULA for stepsize=\{0.1, 0.01\} all experiments show that the algorithm explodes (the moments exceed the infinity value of the computer) so the need to use an alternative becomes apparent.
Here we present a boxplot which describes the the second moment of the first coordinate for different values.
The second moment of each coordinate is given by
\[\mathbb{E}\left[X_i^2\right]=d^{-1} \int_{\mathbb{R}_{+}} r^2 \nu(r) \mathrm{d} r / \int_{\mathbb{R}_{+}} \nu(r) \mathrm{d} r, \quad \nu(r)=r^{d-1} \exp \left\{\left(r^2 / 2\right)-\left(r^4 / 4\right)\right\}\]
and is estimated by a random walk of $10^7$ samples as $\mathbb{E}[X_i^2]=0.104.$
The figure below shows the behaviour of the algorithm for different stepsizes. One can see that the for stepsize=\{0.1, 0.01\} the error is of order $10^{-1}$ while for stepsize =0.001 the error is of order $10^{-2}.$ 

We also present a similar figure for the well-known TULA algorithm devoleped in \cite{tula}. We can see that the TULA algorithm is not as efficient as wd-TULA for large stepsize as it gives an error of approximately 1.1 
but performs better for stepsize 0.1 (gives error of order $10^{-2}$) and similarly to wdTULA for $\lambda=0.001.$

%----------------------------------------------------------------------
%% Numerics figure begins here

\begin{figure}[t]
\centering
\includegraphics[height=30ex]{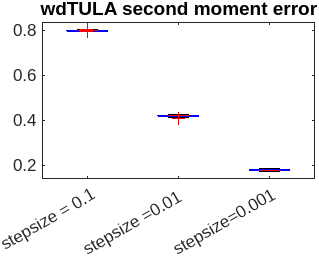}
\hspace{4em}
\includegraphics[height=30ex]{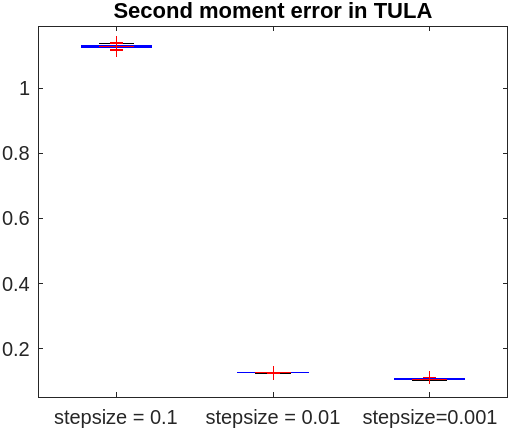}
\caption{Performance of \eqref{eq-wdTULA} compared to TULA (left and right respectively); lower values are better.}
\label{fig:test}
\end{figure}

%% Numerics figure ends here
%----------------------------------------------------------------------

%**********************************************************************
%***    APPENDICES
%**********************************************************************
\appendix
\numberwithin{lemma}{section}	% for numbering  in the appendix
\numberwithin{proposition}{section}	% for numbering  in the appendix
\numberwithin{equation}{section}	% for numbering in the appendix

%----------------------------------------------------------------------
%% APP: PRELIMS
%----------------------------------------------------------------------
\section{Preliminary steps and lemmas}
\label{app:main steps}
%----------------------------------------------------------------------
%%% APP: Main steps
%----------------------------------------------------------------------
% !TEX root = ../Main.tex

\subsection{Moment bounds for \eqref{eq-wdTULA}}
In order to prove the moment bounds for our algorithm, the following properties of our the tamed coeffient will play a pivotal role.
\begin{lemma}
\label{alg-diss}
For all $x\in\R^{d}$, we have
\begin{equation}
\langle \hl(x),x\rangle
	\geq A_{1} \abs{x}^{a} - B_{1}
\end{equation}
where $A_{1} = A/2$ and $B_{1} = \max\{A_{1},B\}$.
In addition, we have
\begin{equation}
\abs{\hl(x)}^{2}
	\leq 4A^{2} \abs{x}^{a} + 2 L^{2}/\lambda + 4 A^{2}
	\quad
	\text{for all $x\in\R^{d}$}.
\end{equation}
%\[\langle x,h_\lambda(x)\rangle\geq A_1|x|^a-B_1.\]
%where $A_1=\frac{A}{2}$ and $B_1=\max\{\frac{A}{2},B\}.$
%In addition,
%\[|\hl(x)|^2\leq 4A^2 |x|^{a} + 2 \frac{L^2}{\lambda}+4A^2  \]
\end{lemma}
\begin{proof}
  Postponed to the proof section.  
\end{proof}

\begin{lemma}\label{expmom1}
Let $M:=\left(2(2d+4A^2+2L^2+A)\right)^\frac{1}{a} $  and $\mu=\frac{Aa M^a}{16(1+M^2)^{1-\frac{a}{2}}}.$ Let $V_\mu(x)=e^{\mu (1+|x|^2)^\frac{a}{2}}.$ There holds, for $\lambda<\min\{1,\frac{A}{4},\frac{2}{\mu^2}\}$ \[\sup_n \E V_\mu (\tn)\leq C_\mu\]
where $C_\mu\leq \mathcal{O}(e^{\mu d}).$
\end{lemma}
\begin{proof}
    Postponed to the Appendix. The main tools used in the proof is the fact that conditionally on the previous step the algorithm is a Gaussian distribution and therefore satisfies a Log-Sobolev. We proceed our proof by using the fact that the function $f(x)=|x|/(1+|x|^2)$ is 1-Lipschitz and we proceed using Herbst argument. The dissipativity and growth condition of our scheme enables to control the key quantity $|x-\hl(x)|^2$.
\end{proof}

\begin{lemma}\label{lemma-algmom}
   Let $p\in \mathbb{N}.$ There holds \[\sup_n |\tn|^{2p}\leq C_p\]
   where $C_p\leq \mathcal{O}(d^p)+2p.$
\end{lemma}

%----------------------------------------------------------------------
%%% Differential inequality
%----------------------------------------------------------------------
\subsection{Establishing a key differential inequality regarding KL- divergence}

The goal of this Section is to establish a differential inequality that will be the basis for our analysis.
We define the continuous-time interpolation of our algorithm given as 
\begin{equation}
    \theta_t=\theta_{k\lambda}-(t-k\lambda) h_\lambda(\theta_{k\lambda}) + {\sqrt{2}}(B_{t}-B_{k\lambda}), \quad \forall t\in [k\lambda,(k+1)\lambda]
\end{equation}
and $\theta_0=\bar{\theta_0}.$

That way \[\mathcal{L}(\theta_{k\lambda})=\mathcal{L}(\bar{\theta}^\lambda_{k}) \quad \forall k \in \mathbb{N}.\]
We define the marginal distribution of $\theta_t$ as $\pt$.
     One notices that, since conditioned on $\theta_{k\lambda}$, $\theta_t$ is a Gaussian its conditional distribution  is given by \[\ptk(x|y)= C e^{-\frac{\sqrt{t-k\lambda}}{2}|x-\mu(t,y)|^2}\]
     where $\mu(t,y)= y-(t-k\lambda)h_\lambda(y)$ and $C$ some normalizing constant. One further notes that, as $\ptk(x|y)$ can be viewed as a distribution of a process satisfying a Langevin SDE with constant drift $-h_\lambda(y)$ and initial condition $y$, i.e
     \[\begin{aligned}
         d\hat{\mu}_t&=-h_\lambda(y)dt + \sqrt{2}dB_t, \quad \forall t \in (k\lambda,(k+1)\lambda]
         \\ \hat{\mu}_{k\lambda}&=y
     \end{aligned}\] it satisfies the following Fokker-Planck PDE:
    \begin{equation}\label{eq-FP}
       \frac{\partial \ptk(x|y)}{\partial t}=div\left(\ptk(x|y) h_\lambda(y)\right) +\Delta_x \ptk(x|y).
    \end{equation}
    Based on rigorous work done in the Appendix we are able to prove the analogous differential in time relative entropy inequality (that originally appeared in \cite{vempala2019rapid} for the vanilla ULA) for our tamed scheme.
    \begin{proposition}\label{divergence}
Let $k\in \mathbb{N}.$
Then, for every $t \in [k\lambda,(k+1)\lambda],$
\[\begin{aligned}
 \frac{d}{dt}  H_{\pi}(\pt)&=-\int_{\mathbb{R}^d}  \langle \pt(x)E \left (h_\lambda(\theta_{k\lambda})\big| \theta_t=x\right)+\nabla \pt(x),\nabla \log\pt(x)-\nabla \log \pi \rangle dx .
\end{aligned}\]
\end{proposition}
\begin{proof}
    See Appendix. The idea of the proof is the same as the one in \citet{lytras2023taming}. We see that our scheme has nice properties such as $||\nabla \hl||\leq \mathcal{O}(\frac{1}{\sqrt{\lambda}})$ and has polynomially growing higher derivatives. Working as in \citet{lytras2023taming} we are able to produce the following inequality and show that in a small neighbourghood of $t$
    \[\pt(x) \leq Ce^{-r|x|^2}\] and $|\nabla \log \pt|$ and $||\nabla^2\log \pt||$ grow polynomially in $x$. This enables some change of the integrals and the derivatives along with application of the divergence theorem. 
    By using Bayes theorem on the conditional Fokker-Planck equation with the divergence theorem we produce the following inequality. 
\end{proof}
  
\newtheorem{Interpolation ineq}[Def1]{Theorem}
\begin{theorem}\label{Interpolation ineq}
Then, for $\lambda<\lambda_{\max}$ and for every $t\in[k\lambda,(k+1)\lambda]$ ,$k\in \mathbb{N},$ there holds
\[\frac{d}{dt}H_{\pi}(\pt)\leq-\frac{3}{4} I_{\pi} (\pt) +   \E | h(\theta_t)-h_{\lambda}(\theta_{k\lambda})|^2.\]
\end{theorem}

\begin{lemma}\label{lemma-onestep}
    There holds \[\E |\theta_t-\theta_{k\lambda}|^{2p} \leq C_{1,p}\lambda^p \]
\end{lemma}
The proof follows by using the growth of property of $\hl$ and the uniform in time moment bounds of the algorithm.

\begin{lemma}\label{tamingerror}
    Then,
    \[\E |h_\lambda(\theta_{k\lambda})-h(\theta_t)|^2\leq C_{err} \lambda \quad \forall k\in \mathbb{N}, \]
    where $C_{err}$ is given explicitly in the proof and depends at most polynomially in the dimension.
\end{lemma}
For this proof we have split the split the difference as follows \[|h(\theta_t)-\hl(\theta_{k\lambda})|^2\leq 2 |h(\theta_t)-h(\theta_{k\lambda})|^2 + 2 |h(\theta_{k\lambda})-\hl(\theta_{k\lambda})|^2.\] The first term is bounded using \cref{ass-pol lip} Lemma \ref{lemma-onestep} and the uniform in time moment bounds of the algorithm (since $\mathcal{L}(\theta_{k\lambda})=\mathcal{L}(\bar{\theta}^\lambda_n)$ and the squared taming error which is of order $\lambda$.\\
Now we are going to provide an inequality between the relative entropy and the Fisher information. When LSI is assumed the connection is immediate.

%----------------------------------------------------------------------
%%% Convergence analysis
%----------------------------------------------------------------------
\subsection{Convergence analysis}

\subsubsection{Convergence under LSI}
\begin{theorem}
    Let $\rho_n$ be the distribution of $n-th$ iterate of the algorithm \eqref{eq-wdTULA}.
Then,  for $\lambda\leq \lambda_{max}$,
\[H_{\pi}(\rho_n)\leq e^{-\frac{3}{2} C_{LSI}\lambda (n-1) } H_{\pi}(\rho_0) + \frac{ \hat{C}}{\frac{3}{2} C_{LSI}}\lambda \]
where $\hat{C}$ is given explicitly in the proof and depends polynomially on the dimension.
\end{theorem}
Using Talagrand's inequality one deduces the following result regarding the convergence in Wasserstein distance. 

\begin{corollary}\label{W2 rate}
    There holds,
\[W_2(\mathcal{L}(\bar{\theta}^\lambda_n),\pi)\leq \frac{\sqrt{2}}{\sqrt{C_{LSI}}}\left( e^{- \frac{3}{4}(C_{LSI}\lambda(n-1)} H_{\pi}(\rho_0) + \sqrt{\frac{\beta \hat{C}}{\frac{3}{2}(C_{LSI})}\lambda}\right).\]
\end{corollary}

\subsubsection{Convergence under PI and weak convexity}
When one does not assume any Log-Sobolev inequality  other tools are needed to describe their connection.
\begin{proposition}\label{eq-inequality h-I}
    There exists $\dot{c}_0>0$ such that 
    \[H_\pi(\pt)\leq \frac{1}{\dot{c}_0} \sqrt{I_\pi(\pt)}.\]
\end{proposition}
The relies on a combination of an inequality between the $W_2$ distance and the Fischer information which stems from the Poincare inequality, combined with the HWI inequality.
\begin{corollary}\label{diff-ineq}
    There holds,
   \[ \frac{d}{dt} H_\pi(\pt)\leq -\dot{c}_0 {H^2_\pi}(\pt) + C_{err}\lambda \quad \forall t\in [k\lambda, (k+1)\lambda]\] 
\end{corollary}
\begin{proof}
    Combining Theorem \ref{Interpolation ineq}, Proposition \ref{eq-inequality h-I} and Lemma \ref{tamingerror} yields the result.
\end{proof}
\begin{proposition}\label{theo-solvdiff}
    Let $\rho_n$ be the $n$-th iteration of our algorithm. Then, there holds
    \[H_\pi(\rho_{k+1})\leq (H_\pi(\rho_{k})^{-1}+\dot{c}_0 \lambda)^{-1} +2C_1 \lambda^2 \]
    \end{proposition}
   The proof proceeds by using a comparison theorem for ODEs to solve the differential inequality.
   By using elementary inequalites we reach a simpler recurrent condition which we iterate over $n$ to reach the following theorem.
    \begin{theorem}\label{theoKL1}
    Suppose that $\lambda$ satisfies the stepsize restrictions given in the moment bounds. In addition, we assume that $\lambda\leq \frac{1}{4\dot{c}_0 C_1}.$ Then,
        There holds \[H_\pi(\rho_n) \leq (1-\frac{\dot{c}_0}{2}\lambda^\frac{3}{2})^n H_\pi(\rho_0) + (1+\frac{4C_1}{\dot{c}_0}) \sqrt{\lambda}\]
    \end{theorem}

        Suppose that $\lambda<\epsilon^2 /4(1+4\frac{C_1}{\dot{c}_0})^2$ Then, $H_\pi(\rho_n)\leq \epsilon$ after $n\geq \mathcal{O}\left(\log (\frac{1}{\epsilon})\frac{1}{\epsilon^3}\right)$ iterations.

    \begin{corollary}\label{cor-othdist}
        Let $\tn$ be the $n-th$ iterate of the algorithm. Then, there holds
        
        \[\|\mathcal{L}(\tn),\pi)-\pi\|_{TV}\leq  \frac{\sqrt{2}}{2} \left( \sqrt{H_\pi(\rho_0)}(1-\frac{\dot{c}_0}{2}\lambda^{\frac{3}{2}})^\frac{n}{2} + \sqrt{(1+4\frac{c_1}{\dot{c}_0})}\lambda^\frac{1}{4}\right) \]
        and
        \[W_1(\mathcal{L}(\tn),\pi)\leq C_W\left( H_\pi(\rho_n)+ {H_\pi(\rho_n)}^\frac{1}{2}\right)\]
    \end{corollary}

%----------------------------------------------------------------------
%%% No regularization
%----------------------------------------------------------------------
\subsection{Proving convergence without \cref{ass11} using a regularized potential}
In the case where \lref{ass11} is missing, we are going to sample from a regularized potential which is close to the original target. The new regularized potential will inherit the important Local Lipschitzness, growth and dissipativity properties of the original and it will also satisfy \lref{ass11} with constant depending on $\lambda$.

We first state some important properties of the new potential which are related to the properties of the original target.

%----------------------------------------------------------------------
%%% Regularized potential
%----------------------------------------------------------------------
\subsubsection{Properties of the regularized potential}
\begin{lemma}\label{reg-properties}
For $\lambda\leq 1$,
    the regularized potential $ \ur$ satisfies \cref{ass-pol lip} with constants independent of $\lambda.$
    As a result, the tamed scheme inherits the $a$-dissipitivity condition, and has at most linear growth.
\end{lemma}
\begin{lemma}\label{lemma-pol-mom2}
    Let $p>1$. There exist, $C_{reg,2p}\leq \mathcal{O}(d^p)$ such that \[\sup_n \E |x^\lambda_n|^{2p}\leq C_{reg,2p}\] 
\end{lemma}
Since the regularized potential satisfies potential satisfies \cref{asm:drift} with different constants independent of $\lambda$ and the tamed scheme inherits the dissipativity condition, the proof follows in the same way as in the unregularized case.
\begin{lemma}\label{lemma poincarereg}
Let $\lambda\leq \frac{ln2}{2R_2^{2r+2}}.$ and $R_2\leq \mathcal{O}(d)$.
    The measure $\pr$ satisfies a Poincare inequality with constant $C_{P,r}^{-1}$ independent of $\lambda$.
\end{lemma}

\begin{proposition}\label{LSI reg}
    The regularized measure satisfies LSI with constant $C_{LSI}^{-1}\leq \mathcal{O}\left((\frac{1}{\lambda})^{\frac{1}{r+1}+\frac{l'}{2r-l'}}\right)$
\end{proposition}
To prove this theorem we use the fact that the regularized measure satisfies a Poincare inequality. By proving that it also satisfies a 2-dissipativity condition and has a lower bounded Hessian, by using some classic theorems depending on Lyapunov functions we show the Poincare inequality can be upgraded to a Log-Sobolev.

\begin{theorem}\label{theo2}
Let $\rd$ be the distribution of the $n$-th iterate of the tamed algorithn with the regularized gradient \eqref{eq-tamedregalg}.  There holds
\[H_\pi(\rd)\leq H_{\pr}(\rd) + \mathcal{O}(\lambda)\leq e^{-\dot{c} \lambda (n-1) } H_{\pr}(\rho_0) + \frac{ \hat{C}}{\Dot{c}}\lambda\] where $\dot{c}=C_{LS}$ given in Lemma \ref{LSI reg} and $\hat{C}$ depends polynomially on the dimension.
\end{theorem}
By using the same arguments as in the unregularized case one reaches a differential inequality. This time we make use of the Log-Sobolev inequality to get a differential inequality for $H_{\pr}(\rho^{reg}_n).$ Then, the proof of the connection between $H_\pi(\rd)$ and $H_\pr(\rd)$ is a simple application of the definition of the relative entropy and the moments of the algorithm and the invariant measure.
\begin{corollary}\label{cor-22}
Let $c_{l,r}:= \frac{r(2+l)}{(r+1)(2r-l)}$
    For $\lambda<\mathcal{O}\left({\epsilon}^\frac{1}{1-c_{l,r}}\right),$ there holds \[H_{\pr}(\rho_n)\leq \epsilon \quad \text{after} \quad  n\geq \mathcal{O}\left(\log(\frac{1}{\epsilon}) (\frac{1}{\epsilon})^{\frac{1+c_{l,r}}{1-c_{l,r}}}\right)\quad \text{iterations}\]
    and \[||\mathcal{L}(\bar{x}^\lambda_n)-\pi||_{TV}\leq \epsilon \quad n \geq \mathcal{O}\left(\log(\frac{1}{\epsilon}^2) (\frac{1}{\epsilon^2})^{\frac{1+c_{l,r}}{1-c_{l,r}}}\right)\quad \text{iterations}\]
    In addition, for $\lambda\leq \mathcal{O}\left({\epsilon}^\frac{2+c_{r,l}}{1-c_{l,r}}\right)$ there holds
    \[W_2(\mathcal{L}(\bar{x}^\lambda_n),\pi)\leq \epsilon \quad \text{after} \quad  n\geq \mathcal{O}\left(\log(\frac{1}{\epsilon}) (\frac{1}{\epsilon})^{(2+c_{r,l})\frac{1+c_{l,r}}{1-c_{l,r}}}\right) \quad \text{iterations}.\]
\end{corollary}

%----------------------------------------------------------------------
%% APP: WD-TULA
%----------------------------------------------------------------------
\section{Proof section}
\label{app:prelims}
%----------------------------------------------------------------------
%%% APP: NON-REGULARIZED
%----------------------------------------------------------------------
% !TEX root = ../Main.tex

\subsection{Proof of preliminary statements}
\begin{proof}[Proof of Lemma \ref{alg-diss}]
It easy to see that if $\langle f(x),x\rangle< 0$, then $\langle f_\lambda(x),x\rangle\geq \langle f(x),x\rangle$ which implies 
that \[\langle h_\lambda(x),x\rangle \geq A|x|^a -B. \]
Suppose that $\langle f(x),x\rangle\geq 0$.
Then, \[\langle \hl(x),x\rangle \geq A\frac{|x|^2}{(1+|x|^2)^{1-\frac{a}{2}}}\geq A|x|^a
\frac{|x|^2}{1+|x|^2}\]
If $|x|>1$ then \[\langle \hl(x),x\rangle \geq \frac{A}{2}|x|^a\]
and if $|x|<1$
\[\langle h_\lambda(x),x\rangle \geq  \frac{A}{2}|x|^a-\frac{A}{2}.\]
To prove the second part one notices that
\[\begin{aligned}
    |\hl(x)|^2&\leq 2(1-\frac{1}{1+\sqrt{\lambda}|x|^{2l}})^2 A^2 \frac{|x|^2}{(1+|x|^2)^{2-a}} + 2(\frac{h(x)}{1+\sqrt{\lambda}|x|^{2l}})^2
    \\&\leq 2 A^2 (1+|x|^2)^{a-1} +2\frac{L^2}{\lambda}
    \\&\leq 2A^2 (1+|x|^2)^\frac{a}{2} +2\frac{L^2}{\lambda}
    \\&\leq 4A^2 + 4A^2 |x|^{a} + 2\frac{L^2}{\lambda}
\end{aligned}  \]

\end{proof}
\subsection{Moment bounds}
\begin{proof}[Proof of Lemma \ref{expmom1}]
The proof starts by noticing that conditioned on $\tn$, $\theta^\lambda_{n+1}$ is a Gaussian, with covariance matrix $\frac{\lambda}{2} I_d$.
Thus conditioned on the previous step since the function $(1+|x|^2)^\frac{1}{2})$ is 1-Lipschitz by t Proposition 5.5.1 in \citet{bakry2014analysis} there holds
so for $\mu^2\leq \frac{2}{\lambda},$ 
\[\begin{aligned}
    \E [V_\mu(\theta^\lambda_{n+1})|\tn]&\leq e^{\mu^2 \lambda} e^{\mu \E[ (1+|\theta^\lambda_{n+1}|^2)^\frac{1}{2}|\tn]} \\&\leq e^{\mu^2 \lambda} e^{\mu \left( 1+ \E[ |\theta^\lambda_{n+1}|^2|\tn]\right)^\frac{1}{2} }
    \\&=e^{\mu^2\lambda} e^{\mu \left (1+ |\tn- \lambda\hl(\tn)|^2 +2\lambda d\right)^\frac{1}{2} }
\end{aligned} \]
where the penultimate step was obtained by Jensen's inequality.
Since  \begin{equation}\label{eq-impbound} 
\begin{aligned}
     |x-\lambda \hl(x)|^2&\leq |x|^2 -2\lambda \langle x,\hl(x)\rangle +\lambda^2|\hl(x)|^2\\&\leq |x|^2 +\lambda(4\lambda A^2-A|x|^a)  +\lambda (A+2L^2) +\lambda^2 4 A^2
\end{aligned}
\end{equation}
    Since for $\lambda\leq \max\{1,\frac{1}{2A}\},$ and $|x|\geq M:= \left(2(2d+4A^2+2L^2+A)\right)^\frac{1}{a}  $
    by \eqref{eq-impbound} one deduces
    \begin{equation}\label{eq-big}
    \begin{aligned}
          ( 1+|x-\hl(x)|^2+2\lambda d)^\frac{1}{2}&\leq (1+|x|^2-\lambda \frac{A}{4} |x|^a)^\frac{1}{2}\\&=(1+|x|^2)^\frac{1}{2} \left(1-\lambda \frac{A}{4} \frac{|x|^a}{(1+|x|^2)}\right)^\frac{1}{2}
          \\&\leq (1+|x|^2)^\frac{1}{2} \left (1- \frac{Aa}{8}\lambda\frac{|x|^a}{(1+|x|^2)}\right)\quad \text{using} \quad  (1-t)^\frac{1}{2}\leq 1-\frac{1}{2}t
          \\&= (1+|x|^2)^\frac{1}{2}-\lambda \frac{A }{8} \frac{|x|^a}{(1+|x|^2)^{1-\frac{1}{2}}}
          \\&= (1+|x|^2)^\frac{1}{2}-\lambda \frac{A }{8} \frac{|x|^a}{(1+|x|^2)^{\frac{1}{2}}}
          \\&\leq (1+|x|^2)^\frac{1}{2} -\lambda \frac{A M^a}{8(1+M^2)^{\frac{1}{2}}}
    \end{aligned}
    \end{equation}
    and the last step was deduced using that the function $g(x)=\frac{x^a}{(1+|x|^2)^{\frac{1}{2}}}$ is increasing for $a\geq 1$ and $x\geq 0$
    Using \eqref{eq-big} one deduces that if $|\tn|\geq M,$
    \[e^{\mu \left(1+|\tn-\lambda \hl(\tn)|^2+2\lambda d\right)^\frac{1}{2}}\leq V_\mu(\tn) e^{-\mu^2 \lambda}\] 
    On the other hand, if $|\tn|\leq M,$ using the inequality $(1+z+y)^\frac{1}{2}\leq {1+z}^\frac{1}{2} +\frac{ y}{2}$ one deduces
    \begin{equation}
        \begin{aligned}
           \mu \left(1+ |\tn-\lambda \hl(\tn)|^2 +2\lambda d\right)^\frac{1}{2}&=\mu\left({1+|\tn|^2 + (2|\tn| |\hl(\tn)| + \lambda^2 |\hl(\tn)|^2 +2\lambda d)}\right)^\frac{1}{2}
            \\&\leq \mu ({1+|\tn|^2})^\frac{1}{2}+ (\frac{\mu}{2} (2|\tn| \lambda|\hl(\tn)| + \lambda^2 |\hl(\tn)|^2 +2\lambda d)^\frac{1}{2}
            \\&\leq \mu({1+|\tn|^2})^\frac{1}{2} +\lambda C_M 
        \end{aligned}
    \end{equation}
    where $C_M\leq C_0+ M^{2l+1}+ M^{4l+2} +2d$ where $C_0$ is an absolute constant independent of the dimension.
    As a result, if $|\tn| \leq M$ 
    \[\begin{aligned}
        e^{\mu \left(1+|\tn-\lambda \hl(\tn)|^2+2\lambda d\right)^\frac{1}{2}}&\leq V_\mu(\tn) e^{C_M\lambda}
    \end{aligned}\]
    which leads to
    \[\begin{aligned}
        \E [V_\mu(\theta^\lambda_{n+1})|\tn]&\leq V_\mu (\tn) e^{(C_M+\mu^2)\lambda}\\&=e^{-\mu^2 \lambda} V_\mu(\tn) + \left(e^{(C_M+\mu^2)\lambda}-e^{-\mu^2 \lambda}\right)V_\mu(\tn) \\&\leq e^{-\mu^2 \lambda} V_\mu(\tn)+ e^{(C_M+\mu^2)\lambda}V_\mu(\tn) \left(1-e^{-(2\mu^2-C_M)\lambda}\right)\\&\leq e^{-\mu^2 \lambda} V_\mu(\tn) + \lambda C
    \end{aligned}
    \]
    where the last step was derived from the inequality $1-e^{-t}\leq t.$
    Putting all together one deduces,
    \[\E V_\mu(\bar{\theta}^\lambda_{n+1})\leq e^{-\lambda \mu^2 n} \E V_\mu(\theta_0) + \bar{C}.\]
\end{proof}
\begin{proof}[Proof of Lemma \ref{lemma-algmom}]
    The proof starts by noticing that the function $g(x)=(ln(x))^{2p}$ is concave for $x\geq e^{2p}.$ As a result, for $n\in \mathbb{N}$
    \[\begin{aligned}
         \E (\mu|\tn|+{2p})^{2p}&\leq \E g(e^{\mu(1+|\tn|^2)^\frac{1}{2}+2p})\\&\leq g(\E e^{\mu(1+|\tn|^2)^\frac{1}{2}+2p}) \quad \text{(Jensen)}
         \\&\leq 2^{2p-1} (ln\E e^{\mu(1+|\tn|^2)^\frac{1}{2}} +2p^{2p}) 
         \\&\leq 2^{2p-1} (ln C_\mu^{2p} +2p^{2p})
    \end{aligned}\]
\end{proof}

\subsection{Rigorous proofs of integral and derivative exchange}
\begin{lemma}\label{lemma-app}
    Let $\lambda<\frac{1}{4(2AC^* +2L+1)^2}$. Then, the following hold:
    Let $k\in \mathbb{N}$ and $t\in[k\lambda,(k+1)\lambda]$.
    Then, there exist $C,r,q$ such that
    \begin{itemize}
        \item \[\pt \leq C e^{-r |x|^2} \quad \forall x\in \mathbb{R}^d\]
        \item  \[ |\nabla \log \pt(x)|\leq C (1+|x|^q) \quad \forall x \in \mathbb{R}^d\]
        \item \[||\nabla^2 \log \pt(x)|\leq C' (1+|x|^{q'}) \quad\forall x \in \mathbb{R}^d.\]
    \end{itemize}
\end{lemma}
\begin{proof}
    Let \[\phi(x)=x-(t-k\lambda)\hl(x)=x -(t-k\lambda) AR(x) -(t-\kappa \lambda) f_\lambda(x)\]
    where $R(x)=\frac{x}{(1+|x|^2)^{1-\frac{a}{2}}}$, $f_\lambda=\frac{h(x)-AR(x)}{g_\lambda}$ and $g_\lambda=\frac{1}{1+\sqrt{\lambda} |x|^{2l}}$.
    Since for the derivatives of $R$ there holds \[||J_R||\leq C^* \] and  for $H:=J_h$, $||H||\leq g_\lambda L\frac{1}{\sqrt{\lambda}},$
   \begin{equation}
   \begin{aligned}
(t-k\lambda) ||A J_R +J_{f_{\lambda}}||&\leq \lambda(A||J_R(x)|| + ||(H-AJ_{R})g_\lambda + \nabla g_\lambda \otimes (h(x)-AR(x))||
\\&\leq \lambda( AC^* +(||H||+AC^*)g_\lambda + |\nabla g_\lambda(x) ||h(x)-AR(x)|)\\&\leq (2AC^* +2L+1) \sqrt{\lambda}\leq \frac{1}{2}.
   \end{aligned}
    \end{equation}
    Thus, \[\frac{1}{2} I_d\leq J_\phi\leq \frac{3}{2} I_d.\]
    In addition, using the fact that the high derivatives of $h$ and $R$ have at most polynomial growth one can easily see that $||J^{(2)}_\phi||$ and $||J^{(3)}_\phi||$ have at most polynomial growth.
    From then, on we proceed with same arguments as in Lemmas A.5-A.7 in \citet{lytras2023taming}.
\end{proof}
  \newtheorem{1exch}[Def1]{Lemma}
    \begin{1exch}\label{1exch}
   \[\E \left(\frac{\partial \ptk(x|\theta_{k\lambda})}{\partial t}\right)=\frac{\partial\pt}{\partial t}(x).\]
    \end{1exch}
\begin{proof}
Analysing the left hand side of the equation one deduces the following:

In a neighbourhood of $t$, for fixed $x$,  $\frac{\partial \ptk(x|y)}{\partial t}$ decays exponentially with $y$ and since $\pkl(y)\leq Ce^{-r|y|^2}$ due to Lemma \ref{lemma-app} one can exchange the derivative with the integral in the following expression
\[\frac{\partial}{\partial t} \int_{\mathbb{R}^d} \pkl(y) \ptk(x|y) dy=\int_{\mathbb{R}^d} \pkl(y) \frac{\partial \ptk(x|y)}{\partial t}dy.\]
Noticing that \[\frac{ \partial \pt}{\partial t} (x)=\frac{\partial}{\partial t} \int_{\mathbb{R}^d} \pkl(y) \ptk(x|y) dy\] and
\[\int_{\mathbb{R}^d} \pkl(y)\frac{\partial \ptk(x|y)}{\partial t}dy=\E \left(\frac{\partial \ptk(x|\theta_{k\lambda})}{\partial t}\right)\] yields the result.
\end{proof}

\begin{lemma}\label{2exch}
\[\E \left( div_x \left (\ptk(x|\theta_{k\lambda}) h_\lambda(\theta_{k\lambda})\right)\right)=div_x \left(\pt(x) \E \left (h_\lambda(\theta_{k\lambda})\big| \theta_t=x\right)\right). \]
\end{lemma}
\begin{proof}
Since $\pt$ decays exponentially with $y$ and for fixed $t$, in a neighbourhood of $x$ $\nabla\ptk(x|y)$ is at most linear in $y$ and $h_\lambda$ has at most linear growth, this enables the interchange of integral and derivative with respect to $x$ in the following expression
\[\int_{\mathbb{R}^d} \pkl(y) div _x\left(\ptk(x|y)h_\lambda(y)\right)dy=div_x \int_{\mathbb{R}^d}\pkl(y) \ptk(x|y) h_\lambda(y) dy\]
Since \[\E \left( div \left (\ptk(x|\theta_{k\lambda}) h_\lambda(\theta_{k\lambda})\right)\right)=\int_{\mathbb{R}^d} \pkl(y) div_x \left(\ptk(x|y)h_\lambda(y)\right)dy\] and due to Bayes theorem
\[\begin{aligned}
div_x \int_{\mathbb{R}^d} \pkl(y) \ptk(x|y)h_\lambda (y) dy&=div _x\int_{\mathbb{R}^d} \pt(x)
\hat{\pi}_{\theta_{k\lambda}|\theta_t}(y|x) h_\lambda(y)dy\\&=div _x\left(\pt(x) \E \left (h_\lambda(\theta_{k\lambda})\big| \theta_t=x\right)\right)
\end{aligned}\]
and the result immediately follows.
\end{proof}

\begin{lemma}\label{3exch}
 \[\E \left(\Delta_x \ptk (x|\theta_{k\lambda} )\right)=\Delta \pt(x).\]
\end{lemma}
\begin{proof}
Noting that by definition \[\E \left(\Delta_x \ptk (x|\theta_{k\lambda} )\right)=\int_{\mathbb{R}^d} \Delta_x(\ptk(x|y))\pkl(y)dy\]
and
\[\Delta_x \pt(x)=\Delta_x \int_{\mathbb{R}^d} \ptk(x|y)\pkl(y)dy\]
it suffices to prove that \[ \int_{\mathbb{R}^d} \Delta_x(\ptk(x|y))\pkl(y)dy=\Delta_x \int_{\mathbb{R}^d} \ptk(x|y)\pkl(y)dy.\]
By simple computations for the Gaussian distribution one deduces that $|\nabla_x \log \ptk(x|y)|$, $\Delta_x \log \ptk(x|y)$ have at most linear growth with respect to $y$ in a neighbourhood of $x$ .
Writing \[\Delta_x \ptk(x|y)=\left(\Delta_x \log \ptk(x|y)+|\nabla_x \log \ptk(x|y)|^2\right)\ptk(x|y)\] one deduces that in a neighbourhood of $x$,
the integrand in the first term is dominated by a function of the form $C(1+|y|^2)e^{-c|y|^2}.$ Applying the dominated convergence theorem enables the exchange of the integral and the Laplacian which completes the proof.
\end{proof}

\begin{corollary}\label{exchres}
\[\frac{\partial \pt}{\partial t}(x)=div_x\left(\pt(x) \E \left (h_\lambda(\theta_{k\lambda})\big| \theta_t=x\right)\right)+  \Delta\pt(x) \quad \forall t \in [k\lambda,(k+1)\lambda] \]
\end{corollary}

\begin{proof}
Taking expectations in \eqref{eq-FP} and combining Lemmas \ref{1exch},\ref{2exch} and \ref{3exch} yields the result.
\end{proof}

\begin{lemma}\label{timechange1}
There exist $C,k,r'>0$ indepent of $x$, uniform in a small neighbourghood of $t$ such that 
\[div_x \left(\pt(x) \E \left (h_\lambda(\theta_{k\lambda})\big| \theta_t=x\right)\right)+\Delta\pt\leq C(1+|x|^k)e^{-r'|x|^2}\]
\end{lemma}
\begin{proof}
Writing, due to Bayes' theorem, 
\[\begin{aligned}
&div_x \left(\pt(x) \E \left (h_\lambda(\theta_{k\lambda})\big| \theta_t=x\right)\right)=\int_{\mathbb{R}^d} \pkl(y) div _x\left(\ptk(x|y)h_\lambda(y)\right)dy\\&\leq Ce^{-c|x|^2+|x|}\int_{\mathbb{R}^d}e^{-r|y|^2}|y| dy
\end{aligned} \] for some $C,c,r>0$ where the last step is a result of the Gaussian expression of the conditional density, the linear growth of $h_\lambda$ and the exponential decay of $\pkl$ given in Lemma \ref{lemma-app}.

For the second term, writing \[\Delta\pt = \pt \left( |\nabla \log \pt|^2+\Delta \log \pt\right)\]  the result follows due to  Lemma \ref{lemma-app}.
\end{proof}

\begin{corollary}\label{cor-timechange}
\[\frac{d}{dt}H_{\pi}(\pt)=\int_{\mathbb{R}^d} \frac{\partial \pt(x)}{\partial t}(1+\log\pt(x)-\log \pi)dx\]
\end{corollary}

\begin{proof}
Noting that $\log\pt,\log\pi$ have polynomial growth, due to Lemma \ref{timechange1},\\ $\frac{\partial \pt(x)}{\partial t}(1+\log\pt(x)-\log \pi)$ can be dominated by an $L^1$ integrable function over small neighbourhood of $t$, thus using the dominated convergence theorem one deduces the exchange of derivative and integration i.e
\[\begin{aligned}
 \int_{\mathbb{R}^d} \frac{\partial \pt(x)}{\partial t}(1+\log\pt(x)-\log \pi(x))dx&=\int_{\mathbb{R}^d}\frac{\partial}{\partial t}\left (\pt(x)\log \frac{\pt(x)}{\pi(x)}\right)dx\\&=\frac{d}{dt}\int_{\mathbb{R}^d} \pt(x)\log \frac{\pt(x)}{\pi(x)}dx\\&=\frac{d}{dt}H_{\pi}(\pt).
\end{aligned}\]
\end{proof}
\begin{proof}[Proof of Corollary \ref{divergence}]
Recall that from Lemma \ref{cor-timechange}, there holds
\begin{equation}
     \frac{d}{dt}  H_{\pi}(\pt)=\int_{\mathbb{R}^d}\frac{\partial \pt(x)}{\partial t}(1+\log\pt(x)-\log \pi)dx.
\end{equation}
Let \[F_t(x)=\pt(x)E \left (h_\lambda(\theta_{k\lambda})\big| \theta_t=x\right)+\nabla \pt(x)\] and \[g_t(x)=1+\log\pt-\log\pi.\]
Recall from Corollary \ref{exchres} that \[\frac{\partial \pt}{dt}(x) =div_x (F_t) (x)\]
Since $\nabla \pt=\pt \nabla \log \pt$ using the   Lemma \ref{lemma-app},\ref{timechange1} one deduces that there exists constants $C$, $q$ ,$r$>0 independent of $x$, uniform in a small neighbourhood of $t$, such that
\begin{equation}\label{eq-divbound}
    \max\{|F_t(x)g_t(x)|,|div(F_t)(x)g_t(x)|,|\langle F_t(x)\nabla g_t(x)\rangle|\}\leq C(1+|x|^q)e^{-r|x|^2}.
\end{equation}
We drop the dependence of the constants on $t$ since we want to integrate with respect to $x.$
Let $R>0$ and $v(x)$ the normal unit vector on $\partial B(0,R)$. Due to \eqref{eq-divbound}
\begin{equation}\label{eq-limit}
    \int_{\partial B(0,R)} \langle g_t(x)F_t(x),v(x)\rangle dx \leq R^d C(1+|R|^q)e^{-r|R|^2}.
\end{equation}
Since $div(F_t) g_t$, $\langle F_t,\nabla_x g_t \rangle$ are integrable (in view of \eqref{eq-divbound}) applying the divergence theorem on $B(0,R)$ there holds
\begin{equation}
    \int_{B(0,R)} div_x(F_t)(x) g_t(x) dx= \int_{\partial B(0,R)} \langle g_t(x)F_t(x),v(x)\rangle dx- \int_{B(0,R)}\langle F_t(x)\nabla_x g_t(x)\rangle dx.
\end{equation}
As a result,
\[\begin{aligned}
\hspace{-20pt} \int_{\mathbb{R}^d} div(F_t)(x) g_t(x)dx&=\lim_{R\rightarrow \infty} \int_{B(0,R)} div_x(F_t)(x) g_t(x) dx\\&=\lim_{R\rightarrow \infty}\left( \int_{\partial B(0,R)} \langle g_t(x)F_t(x),v(x)\rangle dx- \int_{B(0,R)}\langle F_t(x)\nabla_x g_t(x)\rangle\right)dx\\&=0-\lim_{R\rightarrow \infty}\int_{B(0,R)}\langle F_t(x)\nabla_x g_t(x)\rangle dx=-\int_{\mathbb{R}^d}\langle F_t(x)\nabla_x g_t(x)\rangle dx.
\end{aligned}\]
\end{proof}
\begin{proof}[Proof of theorem \ref{Interpolation ineq}]
Using Proposition \ref{divergence}, for all $t \in [k\lambda,(k+1)\lambda],$
    \[\begin{aligned}
\hspace{-20pt} \frac{d}{dt}H_{\pi}(\pt)&=-\int_{\mathbb{R}^d}  \langle \pt(x)E \left (h_\lambda(\theta_{k\lambda})\big| \theta_t =x \right)+\nabla \pt(x),\nabla \log\pt(x)-\nabla \log \pi(x) \rangle dx
 \\&=-\int_{\mathbb{R}^d} \pt(x) \langle E \left (h_\lambda(\theta_{k\lambda})\big| \theta_t=x\right)+\nabla \log \pt(x),\nabla \log\pt(x)-\nabla \log \pi(x) \rangle dx
 \\&=-\int_{\mathbb{R}^d} \pt(x) \langle  E \left (h_\lambda(\theta_{k\lambda})\big| \theta_t=x\right)+\nabla \log \pi,\nabla \log\pt(x)-\nabla \log \pi(x) \rangle dx
 \\&-\int_{\mathbb{R}^d} \pt |\nabla \log\pt(x)-\nabla \log \pi(x) |^2 dx
 \\&=-I_{\pi}(\pt)-\int_{\mathbb{R}^d} \pt(x) \langle  E \left (h_\lambda(\theta_{k\lambda})-h(x)\big| \theta_t=x\right),\nabla \log\pt(x)-\nabla \log \pi(x) \rangle dx
 \\&=-I_{\pi}(\pt)-\int_{\mathbb{R}^d} \pt(x) \langle  E \left (h_\lambda(\theta_{k\lambda})-h(\theta_t))\big| \theta_t=x\right),\nabla \log\pt(x)-\nabla \log \pi(x) \rangle dx
 \\&\leq - I_{\pi}(\pt)+\int_{\mathbb{R}^d} \pt(x) \left | E \left (h_\lambda(\theta_{k\lambda})-h(\theta_t))\big| \theta_t=x\right)\right|^2dx +\frac{1}{4}I_{\pi}(\pt)
 \\&=-\frac{3}{4}I_{\pi}(\pt) +  \int_{\mathbb{R}^d} \pt(x)\left |\int_{\mathbb{R}^d} 
\hat{\pi}_{\theta_{k\lambda}|\theta_t}(y|x) (h_\lambda(y)-h(x))dy \right |^2 dx\\&\leq 
-\frac{3}{4}I_{\pi}(\pt)+  \int_{\mathbb{R}^d} \pt(x) \int_{\mathbb{R}^d} 
\hat{\pi}_{\theta_{k\lambda}|\theta_t}(y|x) \left|h_\lambda(y)-h(x)\right|^2dy dx
 \\&=-\frac{3}{4}I_{\pi}(\pt)+\E | h_\lambda(\theta_{k\lambda})-h(\theta_t)|^2
\end{aligned}\]
where the first inequality was obtained using Young inequality and the second using Jensen's.
\end{proof}
\begin{proof}[Proof of Lemma \ref{lemma-onestep}]
    Let $t \in[ k\lambda,(k+1)\lambda].$
    First of all, one needs to bound the one step error $\E |\theta_t-\theta_{k\lambda}|^{2p}$ for different values of $p\in \mathbb{N}$.
    \[\begin{aligned}
        \E |\theta_t-\theta_{k\lambda}|^{2p}&\leq 2^{2p} \lambda^{2p}\E |h_\lambda(\theta_{k\lambda})|^{2p} +2^{p} \lambda^p \E |Z|^{2p}
        \\&\leq 2^{p} \lambda^{p}\E \left(4A^2|\theta_{k\lambda}|^a+4(L^2+A^2)\right)^{p}+ 2^p\lambda^{p}\E |Z|^{2p}
        \\&\leq \lambda^p C_{1,p}
    \end{aligned}\]
    where $C_{1,p}=\mathcal{O}\left(d^{p(2l+1)}\right),$ which is derived by the moment bounds of the Gaussian, the fact that $\mathcal{L}(\theta_{k\lambda})=\mathcal{L}(\bar {\theta}^\lambda_{k})$ and the moment bounds of the algorithm.
\end{proof}
\begin{proof}[Proof of Lemma \ref{tamingerror}]
    For every $x\in \mathbb{R}^d$,
    \[|h_\lambda(x)-h(x)|=\left|(h(x)-A\frac{x}{(1+|x|^2)^{1-\frac{a}{2}}})(1-\frac{1}{1+\sqrt{\lambda}|x|^{2l}})\right|^2\leq \lambda \left|(|h(x)|+|x|)|x|^{2l}\right|^2\]
    so
    \begin{equation}\label{eq-tamingerr}
        \E \left|h(\theta_{k\lambda})-h_\lambda(\theta_{k\lambda})\right|^2\leq \lambda \E \left|\left(|h(\Bar{\theta}_k)|+|\Bar{\theta}_k|\right)|\Bar{\theta}_k|\right|^2\leq  16\left( L^2 (\bar{C}_{4l}+1)+ \bar{C}_{2l+1}\right) \lambda.
    \end{equation}
    where the constants are given in Lemma \ref{lemma-algmom}.
    In addition, using \cref{ass-pol lip} one deduces that
    \begin{equation}\label{eq-honestep}
        \begin{aligned}
            \E |h(\theta_{k\lambda})-h(\theta_t)|^2 &\leq \E (1+|\theta_{k\lambda}+|\theta_t|)^{2l'} |\theta_t-\theta_{k\lambda}|^2 
            \\&\leq  \sqrt{3^{4l'}(1+ \E |\theta_{k\lambda}^{4l'} +\E |\theta_{k\lambda}-\theta_t|^{4l'}) } \sqrt{\E |\theta_{k\lambda}-\theta_t|^4} 
            \\&\leq \sqrt{3^{4l'}}  \sqrt{1+ \lambda^{2l'} C_{1,2l'} +\sup_n \E |\tn|^{4l'}}\lambda \sqrt{C_{2,p}} \quad (\text{Lemma \ref{lemma-onestep} }
            \\&\leq  \sqrt{1+ \lambda^{2l'} C_{1,2l'} +C_{2l'}}\lambda \sqrt{C_{2,p}}
        \end{aligned}
    \end{equation}
    where the last step was derived by Lemma \ref{lemma-algmom}.
    Combining \eqref{eq-tamingerr} and \eqref{eq-honestep}, yields the result.
\end{proof}
\begin{proof}[Proof of Theorem \ref{thm:wd-TULA-LSI}]
    \[\begin{aligned}
 \frac{d}{dt}H_{\pi}(\pt) &\leq -\frac{3}{4} I_{\pi}(\pt) +   \beta \E |\hl(\theta_{k\lambda}-h(\theta_t)|^2
\\&\leq -\Dot{c} H_{\pi} (\pt) + 2\beta \E |\hl(\theta_{k\lambda})-h(\theta_{k\lambda})|^2 + 2 \beta \E | h(\theta_{k\lambda})-h(\theta_t)|^2
\\&\leq -\Dot{c} H_{\pi} (\pt) + \beta \hat{C} \lambda
\end{aligned}\]
where $\hat{C}=2C_{onestep}+2C_{tam}$
where the first term has been bounded using the Log-Sobolev inequality and the rest of the terms using the one-step error in Lemma \ref{lemma-onestep} and the taming error in Lemma \ref{tamingerror}.
Splitting the terms  one obtains
\[\begin{aligned}
 \left( \frac{d}{dt}H_{\pi}(\pt)+\Dot{c} H_{\pi} (\pt)\right) e^{\Dot{c}t} \leq e^{\Dot{c}t} \beta \hat{C} \lambda
\end{aligned}\]
Integrating over $[k\lambda,t]$ yields
\[\begin{aligned}
 e^{\Dot{c}t}H_{\pi}(\pt)- e^{\Dot{c}k\lambda} H_{\pi}(\hat{\pi}_{k\lambda})\leq \frac{\beta \hat{C}}{\Dot{c}}\lambda (e^{\Dot{c} t}-e^{\Dot{c}k\lambda})
\end{aligned}\]
which implies
\begin{equation}
    H_{\pi}(\pt) \leq e^{\Dot{c}(k\lambda-t)} H_{\pi}(\hat{\pi}_{k\lambda}) +\frac{\beta \hat{C}}{\Dot{c}}\lambda (1-e^{\Dot{c}(k\lambda-t)}).
\end{equation}
Setting $t=n\lambda$ and $k=(n-1)$ leads to
\[H_{\pi}(\hat{\pi}_{n\lambda})\leq e^{-\Dot{c}\lambda } H_{\pi}(\hat{\pi}_{(n-1)\lambda}) +  \frac{\beta \hat{C}}{\Dot{c}}\lambda (1-e^{-\Dot{c}\lambda})\]
so by iterating over $n$,
\[H_{\pi} (\hat{\pi}_{n\lambda})\leq e^{-\dot{c} \lambda (n-1) } H_{\pi}(\pi_0) + \frac{\beta \hat{C}}{\Dot{c}}\lambda\]
which completes the proof.
\end{proof}

\begin{proof}[Proof of Lemma \ref{eq-inequality h-I}]
    Let $f=\pi/\pt$
    Then,
    \begin{equation}\label{eq-1}
        \begin{aligned}
    \frac{1}{2}\int(\sqrt{\pt}-\sqrt{\pi})^2 dx &\leq \left(1-\E_\nu (\sqrt{f})\right)\left(1+\E_\nu (\sqrt{f})\right)
    \\&\leq 1-\left(\E_\nu (\sqrt{f})\right)^2
    \\&= \E_\nu((\sqrt{f})^2)-\left(\E_\pi (\sqrt{f})\right)^2
    \\&=Var_\pi (\sqrt{f})
    \\&\leq \frac{1}{C_P} \E _\pi |\nabla \sqrt{f}|^2
    \\&\leq \frac{1}{4C_P} \E_\pi |\nabla f|^2/f 
    \\&= \frac{1}{4C_P} \int\left( \pi f |\frac{\nabla f}{f}|^2\right)dx
    \\&= \frac{1}{4C_p} \int\left( \pt |-\nabla \log \frac{\pt}{\pi} |^2\right)dx
    \\&= \frac{1}{C_p}I_\pi(\pt).
    \end{aligned}
    \end{equation}
    In addition, since both $\pt$ and $\pi$ have finite polynomial moments, there holds
  \begin{equation}\label{eq-WI}
       \begin{aligned}
        W_2^2(\mathcal{L}(\theta_t),\pi)&=2\int |x|^2 |\pi(x)-\pt(x)|dx
        \\&\leq 2 \left(\int |x|^4 (\sqrt{\pi}+\sqrt{\pt})^2 dx\right)^\frac{1}{2}
        \left(\int (\sqrt{\pt}-\sqrt{\pi})^2 dx\right)^\frac{1}{2} 
        \\&\leq 
        4 (\sqrt{ \E_{\pt} |x|^4}+\sqrt{\E_\pi |x|^4})\sqrt{I_\pi(\pt)} \quad \text{derived from \eqref{eq-1}}
        \\&32 (\sqrt{\sup\E |\tn|^4} +\sqrt{\E |\theta_t-\theta_{k\lambda}|^4} +\sqrt{\E_\pi |x|^4})\sqrt{I_\pi(\pt)}
        \\&\leq C \sqrt{I_\pi(\pt)}
    \end{aligned}
  \end{equation} 
    where the last step was derived from Lemmas \ref{lemma-onestep} and \ref{lemma-algmom}
    We are going to use our assumption to connect the relative entropy to $W_2$ distance.
    Since \cref{ass11} holds and $\pi$ has finite second moments, the HWI can be applied, so 
    \begin{equation}\label{eq-fromHWI}
    \begin{aligned}
        H_\pi(\pt)&\leq \sqrt{I_\pi(\pt)}W_2(\mathcal{L}(\theta_t),\pi) +\frac{\kappa}{2} W_2^2(\mathcal{L}(\tn),\pi)\\&\leq \sqrt{2}(\sqrt{\E_\pi} |x|^2 +\sqrt{\E |\theta_{k\lambda}-\theta_t|^2}+\sqrt{\E |\tn|^2}) \sqrt{I_\pi(\pt)} + \frac{\kappa}{2} W_2^2(\mathcal{L}(\tn),\pi).
    \end{aligned}
    \end{equation}
    Combining \eqref{eq-WI} with \eqref{eq-fromHWI} yields the result.
\end{proof}
 \begin{proof}[Proof of Proposition  \ref{theo-solvdiff}]
        Let $\phi(t,x)=-\dot{c}_0 x^2+ k_2$ where $k_2:=2C_1\lambda.$
        Then, from Corollary \ref{diff-ineq} there holds \[\frac{d H_\pi(\pt)}{dt}<\phi( H_\pi(\pt),t).\]
        Let $\delta<H_\pi^{-1}(\rho_k)/2$
        Setting $g_\delta(t)=\left(H_\pi(\rho_k)^{-1}-\delta+\dot{c}_0 (t-k\lambda)\right)^{-1}+k_2(t-k\lambda)$ one deduces 
        \begin{equation}
            g_\delta'(t)=-\dot{c}_0 \left(H_\pi(\rho_k)^{-1}-\delta+\dot{c}_0 (t-k\lambda)\right)^{-2} +k_2
        \end{equation}
        Since $(H_\pi^{-1}(\rho_k)-\delta)^2\leq \left(H_\pi(\rho_k)^{-1}-\delta+\dot{c}_0 (t-k\lambda)\right)^{2} $ one obtains \begin{equation}\label{eq-diff-ineq}
            g_\delta'(t)-\phi(g_\delta(t),t)\geq 0> \frac{d H_\pi(\pt)}{dt}-\phi( H_\pi(\pt),t) \quad \forall t \in[k\lambda, (k+1)\lambda]
        \end{equation}
        Using \eqref{eq-diff-ineq} and the fact that $g_\delta(k\lambda)=\left(H_\pi(\rho_k)^{-1}-\delta\right)^{-1}>H_\pi(\rho_k)=H_\pi(\hat{\pi}_{k\lambda})$
        By comparison theorem for differential inequalities, see \citet{mcnabb1986comparison}, there holds
        \begin{equation}
      (H_\pi(\rho_{k})^{-1}+\dot{c}_0 \lambda)^{-1} +2C_1 \lambda^2    =\lim_{\delta \rightarrow 0^+} g_\delta((k+1)\lambda)\geq H_\pi(\rho_{k+1})
        \end{equation}
        \end{proof}
         \begin{proof}[Proof of theorem \ref{theoKL1}]
    \newcommand{\Hk}{H_\pi(\rho_{k})}
    \newcommand{\Hko}{H_\pi(\rho_{k+1})}
        We begin the proof by noticing that \begin{equation}\label{eq-ind}
            H_\pi(\rho_n)\leq \frac{1}{\dot{c}_0\lambda} \quad \forall n.
        \end{equation}
        This will be done by induction. Since for $\lambda<\lambda_{max},$ \[H_\pi(\rho_0)\leq \frac{1}{\dot{c}_0 \lambda}\] it holds for $n=0.$
        Suppose that \begin{equation}\label{eq-k
        }
             H_\pi(\rho_k)\leq \frac{1}{\dot{c}_0\lambda}
        \end{equation}
        Then \[\begin{aligned}
            H_\pi(\rho_{k+1})\leq \frac{1}{\dot{c}_0\lambda} \left(\dot{c}_0 \lambda\Hk (1+\dot{c}_0\lambda \Hk)^{-1}\right) +2C_1\lambda^2
        \end{aligned}\]
        Since the function $\phi(x)=\frac{x}{1+x}$ is increasing, then $\phi(\dot{c}_0\lambda \Hk)<\phi(1)$ so
        \[\Hko\leq \frac{1}{2\dot{c}_0\lambda} + 2C_1\lambda^2 \leq \frac{1}{\dot{c}_0\lambda}.\]
        which proves \eqref{eq-ind} by induction.\\
        We proceed with two cases: \\
        \textbf{Case 1}: $H_\pi(\rho_{k_0}) \geq \frac{4c_1}{\dot{c}_0}\sqrt{\lambda} \quad \forall k_0\leq k$:\\
        Making use of \eqref{eq-ind} and the inequality $\frac{1}{x+1}\leq (1-\frac{x}{2})$ for $x\leq 1$, one obtains
        \[\Hko \leq \Hk (1-\frac{\dot{c}_0}{2} \lambda \Hk) +2C_1\lambda^2\leq \Hk (1-2C_1 \lambda^\frac{3}{2}) +2C_1\lambda^2\]
        Summing over $k$ one deduces 
        \begin{equation}
            \Hk \leq H_\pi(\rho_0) (1-2C_1) \lambda^\frac{3}{2})^k +\sqrt{\lambda}. 
        \end{equation}
        \textbf{Case 2}: There exist $k_0\leq k$ such that $H_\pi(k_0)\leq \sqrt{\frac{4c_1}{\dot{c}_0}}\sqrt{\lambda}.$
Suppose that $\frac{4c_1}{\dot{c}_0}\sqrt{\lambda}\geq H_\pi(\rho_{k_0})\geq  \frac{1}{2}\frac{4c_1}{\dot{c}_0}\sqrt{\lambda}.$ Then,
\[H_\pi(\rho_{k_0+1})\leq H_\pi(\rho_{k_0}) -\frac{\dot{c}_0}{2}\lambda  H_\pi^2(\rho_{k_0}) +C_1\lambda^2\leq H_\pi(\rho_{k_0}) \]
    On the other hand, if $H_\pi(k_0)\leq  \frac{1}{2}\frac{4c_1}{\dot{c}_0}\sqrt{\lambda}.$ it is easy to see that
    $H_\pi(\rho_{k+1})\leq \frac{4c_1}{\dot{c}_0}\sqrt{\lambda}.$
    
    This implies that \[\exists k_0<k: H_\pi(\rho_{k_0})\leq \frac{4c_1}{\dot{c}_0}\sqrt{\lambda} \implies \Hk \leq \frac{4c_1}{\dot{c}_0}\sqrt{\lambda}\]
    Combining case 1 and case 2  together yields the result.
    \end{proof}
    \begin{proof}[Proof of Corollary \ref{cor-othdist}]
        For the bound in total variation, using Theorem \ref{theoKL1} and Pinsker's inequality gives the result.\\
        For the bound in $W_1$ distance, using Lemma \ref{expmom1} and Corollary 2.3 in \cite{bolley2005weighted}, one deduces that
        \[C_W:=\frac{2}{\mu}(\frac{3}{2}+ \log \E e^{\mu |\tn|})<\infty\] and
        \[W_1(\mathcal{L}(\tn),\pi)\leq C_W\left( H_\pi(\rho_n)+ {H_\pi(\rho_n)}^\frac{1}{2}\right).\]
Applying the bound on $C_W$ in Lemma \ref{expmom1} and Theorem \ref{theoKL1} yields the result.
    \end{proof}

%----------------------------------------------------------------------
%% APP: REG-TULA
%----------------------------------------------------------------------
\section{Proofs for the convergence of regularized scheme}
\label{app:potential}
%----------------------------------------------------------------------
%%% APP: REGULARIZED
%----------------------------------------------------------------------
% !TEX root = ../Main.tex

\subsection{Properties for the regularized potential}
    \begin{proof}[Proof of Lemma \ref{reg-properties}]
It is easy to see that the function  $G(x)=(r+1) |x|^{2r}$ is Locally Lipschitz since \[J_ G= (r+1) |x|^{2r} I_d + (r+1)r x^tx |x|^{2r-2}\] then,
\[||J_G(x)||\leq (r+1)^2|x|^{2r}.\]
Using the mean value theorem \[|\lambda G(x)-\lambda G(y)|\leq \lambda \int_0^1 ||J_G(tx+(1-t)y|||x-y|dt\leq \lambda(r+1)^2(1+|x|+|y|)^{2r} |x-y|.\]
As a result, \[|\nabla\ur(x)-\nabla \ur(y)|\leq ((r+1)^2+L)(1+|x|+|y|)^{2r}|x-y| \quad \forall x,y \in \mathbb{R}^d.\]
It is also easy to see that the higher derivatives of $\ur$ have polynomial growth less than $2r+1$.
With respect to the dissipativity it is easy to see that \[
\langle \nabla\ur(x),x\rangle \geq \langle \nabla u(x),x\rangle,
\]
so \cref{ass-2dissip} is satisified with the same $A$ and $b$.
For the tamed scheme, by it is definition it easy to see that \[|\hr|\leq A +\sqrt{\lambda} +A|x|^\frac{a}{2} +\frac{(L+1)}{\sqrt{\lambda}}.\]
    \end{proof}
    \begin{proof}[Proof of Lemma \ref{lemma poincarereg}]
    The proof starts by noticing that there exists $R_1$ depending on $A,B$ of \cref{ass-2dissip} such that \[\langle \nabla \ur (x),x\rangle \geq \frac{A}{2} |x| \quad \forall |x|\geq R_1.\]
    In addition, picking a smooth Lyapunov function $W\geq 1$ such \[W=e^{\frac{A}{4}|x|} \quad \forall |x|\geq R_1,\]
    one deduces that for generator of the Langevin SDE with drift coeffient the regularized gradient,
    \[\begin{aligned}
        LW(x)&=\Delta W(x)-\langle \nabla W(x),\nabla \ur(x)\leq  \frac{A}{4}W \left( \frac{d-1}{|x|} +\frac{A}{4}-\langle \nabla \ur(x),x\rangle\} \right)
        \\&\leq \frac{A}{4}W (\frac{d-1}{|x|} +\frac{A}{4} -\frac{A}{4}|x|)
    \end{aligned}\]
    So there exists $R_0\leq \mathcal{O}({d})$ such that 
    \[LW \leq -\theta W \quad \forall |x|\geq R_2:=\max\{R_0,R_1\}.\]
    Setting $B=B(0,R_2)$ and $B_2=B(0,R_2+2)$.
Let a smooth function $\chi=\psi(|x|)$ (see Lemma B.13 of \citet{li2020riemannian} for the construction) such that  $\chi=0$ on $B$ nad $\chi=1 $ on $B_2^c$ and $|\nabla \chi|\leq 1.$
   
    \begin{equation}
        \begin{aligned}
\int \frac{-L W}{W} f^2 d \pr & =\int \Gamma\left(\frac{f^2}{W}, W\right) d \pr \\
& =2 \int \frac{f}{W} \Gamma(f, W) d \pr-\int \frac{f^2}{W^2} \Gamma(W, W) d \pr \\
& =-\int\left|\frac{f}{W}  \nabla W- \nabla f\right|^2 d \pr+\int \Gamma(f, f) d \pr 
\\&\leq \int \Gamma(f, f) d \pr
\end{aligned}
    \end{equation}
     Writing for a smooth $f$,
    $$
\begin{aligned}
\int f^2 d \pr & =\int(f(1-\chi)+f \chi)^2 d \pr \\
& \leq 2 \int f^2(1-\chi)^2 d \pr+2 \int f^2 \chi^2 d \pr \\
& \leq \frac{2}{\theta} \int \frac{-L W}{W} f^2(1-\chi)^2 d \pr+2 \int_{B_2} f^2 d \pr \\
& \leq \frac{2}{\theta} \int \Gamma(f(1-\chi), f(1-\chi)) d \pr+2 \int_{B_2} f^2 d \pr
\end{aligned}
$$
. Since $\Gamma(f g, f g) \leq 2\left(f^2 \Gamma(g, g)+g^2 \Gamma(f, f)\right)$, we get:
\begin{equation}\label{eq-prePoinc}
    \begin{aligned}
\int f^2 d \pr & \leq \frac{4}{\theta} \int \Gamma(f, f) d \pr+\frac{4}{\theta} \int f^2 \Gamma(\chi, \chi) d \pr+2 \int_{B_2} f^2 d \pr \\
& \leq \frac{4}{\theta} \int \Gamma(f, f) d \pr+\left(\frac{4}{\theta}+2\right) \int_{B_2} f^2 d \pr
\end{aligned}
\end{equation}
Applying the previous inequality for  $\tilde{f}= f-\int_{B_2} fd\pr$  and using the fact that \[Var_{\pr} (f)\leq \int \tilde{f}^2d\pr\] yields
\begin{equation}\label{eq-poinfin}
\begin{aligned}
    Var_{\pr}(f)&\leq \int \tilde{f}^2d\pr\leq \frac{4}{\theta} \int \Gamma(\tilde{f}, \tilde{f}) d \pr+\left(\frac{4}{\theta}+2\right) \int_{B_2} \tilde{f}^2 d \pr
    \\&=\frac{4}{\theta} \int \Gamma(f, f) d \pr+\left(\frac{4}{\theta}+2\right)\int_{B_2} \tilde{f}^2 d \pr
\end{aligned}
\end{equation}
  
  When restricted to the ball $B_2$ $\ur$ is a bounded petrubation  $u$ on the same ball since \[|u(x)-\ur(x)|\leq \lambda (R_2+2)^{2r+2} \quad \forall x \in B_2.\] Using Hooley-Strook pertubation theorem one deduces that $\pr$ satisfies Poincare inequality when resticted to $B_2$ with constant $k_{B_2}^{-1}\leq e^{2\lambda (R_2+2)^{2r+2}} C_P^{-1}\leq 3 C_P{-1}.$  
  Thus, \[\int_{B_2} \tilde{f}^2 d \pr\leq k_{B_2} \int \Gamma(f, f) d \pr. \]
  Applying this to \eqref{eq-poinfin} completes the proof.
\end{proof}
\begin{lemma}\label{lemma-reg}
The function $u$ given by $u_{r,\lambda}(x):=u(x)+\lambda |x|^{2r+2}$ satisfies
    \[\langle \hr(x)- \hr(y),x-y\rangle \geq \left(c_1(|x|^{2r} +|y|^{2r})-c_2(|x|^{l'}+|y|^{l'})-c_3\right) |x-y|^2 \quad \forall x,y\in \mathbb{R}^d.\]
    where $c_1:=\lambda (r+1)$, $c_2=c_3=L$.
    \end{lemma}
    The proof follows by using \cref{ass-pol lip} and the fact that the regularized term $2r$ dominates $l$ for large values. This will yield a lower bound for the minimum eigenvalue of $\nabla^2 \ur.$
        \begin{proof}
        Let $f(x)=|x|^{2r+2}.$ Then $\nabla f(x)= 2(r+1) |x|^{2r} x.$
        Writing \[\begin{aligned}
            \langle \nabla f(x)-\nabla f(y),x-y\rangle&=\langle \nabla f(x),x\rangle +\langle\nabla f(y),y\rangle-\langle \nabla f(x),y\rangle - \langle \nabla f(y),x \rangle\\&= (2r+2) (|x|^{2r+2}+|y|^{2r+2})- (r+1)(|x|^{2r}+|y|^{2r})2\langle x,y\rangle\\&=(2r+2) (|x|^{2r+2}+|y|^{2r+2})\\&+ (r+1)(|x|^{2r}+|y|^{2r})\left( |x-y|^2 -|x|^2-|y|^2\right)
            \\&=(r+1)\left(|x|^{2r+2}+|y|^{2r+2}-|x|^{2r}|y|^{2}-|y|^{2r}|x|^{2}\right) \\&+ (r+1)(|x|^{2r}+|y|^{2r})|x-y|^2.
        \end{aligned}\]
        Since \[\begin{aligned}
            |x|^{2r+2}+|y|^{2r+2}-|x|^{2r}|y|^{2r+2}-|y|^{2r}|x|^{2r+2}&=|x|^{2}(|x|^{2r}-|y|^{2r})-|y|^2(|x|^{2r}-|y|^{2r})\\&=
        (|x|^2-|y|^2)(|x|^{2r}-|y|^{2r})\\&\geq 0,
        \end{aligned}\]
        one deduces
        \begin{equation}
            \langle \nabla f(x)-\nabla f(y),x-y\rangle \geq (r+1)(|x|^{2r}+|y|^{2r})|x-y|^2 \quad \forall x,y \in \mathbb{R}^d.
        \end{equation}
       Noting that by the gradient local Lipschitz assumption on $g$, there holds
        \[\langle \nabla u(x)- \nabla u(y),x-y\rangle\geq -L(1+|x|^l+|y|^l) |x-y|^2 \quad \forall x,y \in \mathbb{R}^d,\] the result immediately follows.
    \end{proof}
\begin{proof}[Proof of Proposition \ref{LSI reg}]
    It is easy to see that the regularized measure satisfies a $2-$ dissipativity condition with constant $A_{reg}=\lambda^{\frac{1}{r+1}}$ i.e
    \begin{equation}
        \langle \hr(x),x\rangle\geq A_{reg}|x|^2-(b+1).
    \end{equation}
    In addition, using Lemma \ref{lemma-reg}, it is easy to see that when $|x|,|y|\geq \frac{L}{r+1}(\frac{1}{\lambda})$ \[\langle \hr(x)- \hr(y),x-y\rangle \geq 0\] so one concludes that
    \[\langle \hr(x)-\hr(y),x-y\rangle \geq -K_\lambda|x-y|^2 \quad \forall x,y \in \mathbb{R}^d\] 
    which leads to 
    \begin{equation}\label{eq-Hess bound}
        \nabla^2 u_{r,\lambda}(x)\geq -K_\lambda I_d \quad \forall x \in \mathbb{R}^d.
    \end{equation}
    Let $W:=e^\frac{A_{reg} |x|^2}{4}.$ Then, since $\nabla W= W \frac{A_{reg}}{2} x$ and $\Delta W\leq (\frac{A_{reg}d}{2} + \frac{A^2_{reg}d}{4})W $ one observes that
    \begin{equation}
        LW=\Delta W -\langle \hr(x),\nabla W\rangle= (\frac{A_{reg}d}{2} + \frac{A^2_{reg}d}{4}-\frac{A_{reg}}{2} |x|^2) W
    \end{equation}
    Since $\pi_{reg}$ also satisfies a Poincare inequality with consant $C_{P,r}$ using Theorem 3.15 in \cite{menz2014poincare},
    it can be upgraded to a Log-Sobolev inequality with constant 
    \begin{equation}\label{eq-LSIreg}
    \begin{aligned}
          \frac{1}{C_{LS}} &\leq 2 \sqrt{\frac{1}{A_{reg}}\left(\frac{1}{2}+\frac{+\frac{A^2_{reg}d}{4}-\frac{A_{reg}}{2}+ \frac{A_{reg}}{2} \pr\left(|x|^2\right)}{C_{P,reg}}\right)}\\&+\frac{K_\lambda}{A_{reg} \lambda}+\frac{K_\lambda\left(\frac{A_{reg}d}{2} + \frac{A^2_{reg}d}{4}+\frac{A_{reg}}{2} \pr\left(|x|^2\right)\right)+2  \frac{A_{reg}}{2}}{C_{P,reg }}
    \end{aligned}
    \end{equation}
      where $\pr(|x|^2)$ is given by Lemma \ref{inv-mom bound}.
\end{proof}
\begin{proof}[Proof of \ref{lemma-pol-mom2}]
    The proof is the same, as in the proof of the moment bounds for the regularized potential, since all we need is the $a-$ dissipativity and the growth condition which are almost the same. It is done through providing first exponential moments and then produce polynomial.
\end{proof}

\begin{lemma}
    There holds \[\E |x_t-x_{k\lambda}|^{2p} \leq \mathcal{O}(\lambda^p)\]
\end{lemma}
\begin{proof}
    The proof follows in the same way is the respective one for the unregulazized algorithm. 
\end{proof}
\begin{lemma}
    There holds \[\E|\hr(x_{k\lambda})-\htl(x_{k\lambda})|^2\leq C^{reg}_{tam} \lambda\]
    where $C^{reg}_{tam}\leq \mathcal{O}(d^{4r+2}).$
\end{lemma}
\begin{proof}
    Writing \[\begin{aligned}
        |\hr(x_{k\lambda})-\htl(x_{k\lambda})|^2 &\leq 2 |h(x_{k\lambda})-\hl(x_{k\lambda})|^2 +2 \lambda^2 \left| (r+1) x_{k\lambda}|x_{k\lambda}|^{2r} (1-\frac{1}{1+\sqrt{\lambda}|x_{k\lambda}|^{2r+1}}\right|^2\\&\leq 2 |h(x_{k\lambda})-\hl(x_{k\lambda})|^2 +2\lambda^2 (r+1)^2 |x_{k\lambda}|^{4r+2}. 
    \end{aligned}\]
    Taking expectations, the first term can be treated as in the proof of Lemma \ref{tamingerror} with the moment bounds in Lemma \ref{lemma-pol-mom2}.
\end{proof}
\begin{lemma}
    There holds \[\E |\hr(x_t)-\hr(x_{k\lambda}|^2\leq C^{reg}_{onestep} \lambda\]
\end{lemma}
\begin{proof}
    The proof follows in the same way as the respective one for the unregularized algorithm with the new moment bounds, and setting the Local Lipschitz constants as in Lemma \ref{reg-properties}.
\end{proof}
\begin{lemma}
Let $x_t$ the continuous time interpolation of the algorithm.
    Then, for $\lambda<\lambda_{\max}$ and for every $t\in[k\lambda,(k+1)\lambda]$ ,$k\in \mathbb{N},$ there holds
\[\frac{d}{dt}H_{\pi}(\pt^{reg})\leq-\frac{3}{4} I_{\pi} (\pt^{reg}) +   \E | \hr(x_t)-\htl(x_{k\lambda})|^2\].
\end{lemma}
\begin{proof}
    Since the regularized potential has the same key properties as the unregularized, the interpolation inequality holds with exactly the same arguments.
\end{proof}
\begin{lemma}\label{inv-mom bound}
    Let $p\in \mathbb{N}.$ There holds
    \[\E_\pi |x|^{2p}\leq C_\pi  \] and
    \[\E_{\pr} |x|^{2p} \leq C_{\pr} \]
    where $C_{\pr}$ and $C_{\pi}$ are $\mathcal{O}(d^{2p}).$
\end{lemma}
\begin{proof}
    Using the fact that both measures satisfy the dissipativity condition with constant $a$ we will proceed with the same arguments. We will show it only for $\pi$.
Since $\pi$ and $\pr$ satisfy a Poincare inequality they have finite polynomial moments of all orders.
    %\[\begin{aligned}
      %  u(x)&= u(0) +\int_0^1 \langle h(tx),x\rangle dt = u(0) + \int_0^\frac{R}{|x|^{\frac{2l+1}{2l}}}\langle h(tx),x\rangle dt+ \int_\frac{R}{|x|^\frac{2l+1}{2l}}^1\langle h(tx),x\rangle dt
      %
    
    \begin{equation}
        |x|^{2p-1} \langle h(x),x\rangle \geq |x|^{2p-1} (A|x|^a -b)\geq   A|x|^{2p}-A -b |x|^{2p-1} .
    \end{equation}
    Setting $V(x)=|x|^{2p}$ one notices that $\nabla V(x)=2p |x|^{2p-1} x$ and $\Delta V= (2pd+ 4(p-1)p)|x|^{2p-2}.$
    Since $\pi$ is the invariant measure of the Langevin SDE with generator \[LV=\Delta V-\langle V,h\rangle,\]
    there holds 
    \[(2pd+ 4(p-1)p)\E_\pi|x|^{2p-2}=\E_\pi \Delta V(x)=\E_\pi \langle V(x),h(x)\rangle\geq 2p \left (A\E_\pi|x|^{2p}-A -b \E_\pi|x|^{2p-1} \right) . \]
    Iterating over $2p$ yields the result.
\end{proof}
\begin{proof}[Proof of Theorem \ref{theo2}]
    Setting $\dot{c}=\frac{3}{2}C_{LSI}$  one obtains
\[\begin{aligned}
 \frac{d}{dt}H_{\pr}(\pt^{reg}) &\leq -\frac{3}{4} I_{\pr}(\pt^{reg}) +    \E |\htl(x_{k\lambda}-\hr(x_t)|^2
\\&\leq -\Dot{c} H_{\pr} (\pt^{reg}) + 2 \E |\htl(x_{k\lambda})-\hr(x_{k\lambda})|^2 + 2  \E | \hr(x_{k\lambda})-\hr(x_t)|^2
\\&\leq -\Dot{c} H_{\pr} (\pt^{reg}) +  \hat{C} \lambda
\end{aligned}\]
where $\hat{C}$ depends polynomially on the dimension,
where the first term has been bounded using the Log-Sobolev inequality and the rest of the terms using the one-step error and taming, as in the unregularized case .
Splitting the terms  one obtains
\[\begin{aligned}
 \left( \frac{d}{dt}H_{\pr}(\pt^{reg})+\Dot{c} H_{\pr} (\pt^{reg})\right) e^{\Dot{c}t} \leq e^{\Dot{c}t}  \hat{C} \lambda
\end{aligned}\]
Integrating over $[k\lambda,t]$ yields
\[\begin{aligned}
 e^{\Dot{c}t}H_{\pr}(\pt^{reg})- e^{\Dot{c}k\lambda} H_{\pr}(\hat{\pi}^{reg}_{k\lambda})\leq \frac{ \hat{C}}{\Dot{c}}\lambda (e^{\Dot{c} t}-e^{\Dot{c}k\lambda})
\end{aligned}\]
which implies
\begin{equation}
    H_{\pr}(\pt^{reg}) \leq e^{\Dot{c}(k\lambda-t)} H_{\pr}(\hat{\pi}^{reg}_{k\lambda}) +\frac{ \hat{C}}{\Dot{c}}\lambda (1-e^{\Dot{c}(k\lambda-t)}).
\end{equation}
Setting $t=n\lambda$ and $k=(n-1)$ leads to
\[H_{\pr}(\hat{\pi}^{reg}_{n\lambda})\leq e^{-\Dot{c}\lambda } H_{\pr}(\hat{\pi}^{reg}_{(n-1)\lambda}) +  \frac{ \hat{C}}{\Dot{c}}\lambda (1-e^{-\Dot{c}\lambda})\]
so by iterating over $n$,
\[H_{\pr} (\hat{\pi}^{reg}_{n\lambda})\leq e^{-\dot{c} \lambda (n-1) } H_{\pr}(\rho_0) + \frac{ \hat{C}}{\Dot{c}}\lambda\]

Noticing that
    \[\begin{aligned}
      \int \log \frac{\rd}{\pi} d\rd &= \int \log \frac{\rd}{\pr} d\rd + \int 
      \log \frac{\pr}{\pi} d\rd 
      \\&= H_{\pr}(\rd) +\int \log \frac{\pr}{\pi} d(\rd-\pi) -H_\pr(\pi)
      \\&\leq H_{\pr}(\rd) + \lambda\E_{\rd} [|x|^{2r+2}] + \lambda \E_\pi[|x|^{2r+2}].
    \end{aligned}\]
    The result follows by Lemma \ref{inv-mom bound}.
\end{proof}

\begin{proof}[Proof of Corollary \ref{cor-22}]
Using Pinsker's inequality and the bound in Theorem \ref{theo2} one obtains the bound in total variation.
Recall that for the  $W_2$ distance, since $\pr$ satisfies a Log-Sobolev inequality then, it satisfies a Talagrand inequality with same constant.

\[\begin{aligned}
    W_2(\mathcal{L}(\bar{x}^\lambda_n,\pi)&\leq W_2(\mathcal{L}(\bar{x}^\lambda_n,\pr) + W_2(\pi,\pr)
    \\&\leq \sqrt{ 2 \CLSI^{-1} (H_\pr(\rd)}+2\CLSI^{-1}\sqrt{I_\pr(\pi)})
\end{aligned}\]
The first term can be bounded by Theorem \ref{theo2} while the second term is \[\sqrt{\int |\nabla\log\pr(x)-\nabla \log \pi(x)|^2 d\pi} \leq \lambda \sqrt{\E_\pi [|x|^{4r+2}]}. \]
Using the bound on the Log Sobolev constant and Lemma \ref{inv-mom bound} leads to the result.
\end{proof}

%----------------------------------------------------------------------
%%% THANKS
%----------------------------------------------------------------------
\section*{Acknowledgments}
\begingroup
\small
%----------------------------------------------------------------------
%%% ACKS
%----------------------------------------------------------------------
% !TEX root = ./Main.tex
%
%
This research was supported in part by 
the French National Research Agency (ANR) in the framework of
the PEPR IA FOUNDRY project (ANR-23-PEIA-0003),
the ``Investissements d'avenir'' program (ANR-15-IDEX-02),
the LabEx PERSYVAL (ANR-11-LABX-0025-01),
MIAI@Grenoble Alpes (ANR-19-P3IA-0003).
PM is also a member of the Archimedes Research Unit, Athena RC, Department of Mathematics, University of Athens.
This research was also supported in part by 
project MIS 5154714 of the National Recovery and Resilience Plan Greece 2.0 funded by the European Union under the NextGenerationEU Program.
\endgroup

%**********************************************************************
%***    BIBLIOGRAPHY
%**********************************************************************
\bibliographystyle{icml}
\bibliography{bibtex/IEEEabrv,Bibliography-PM}

\begin{thebibliography}{34}
\providecommand{\natexlab}[1]{#1}
\providecommand{\url}[1]{\texttt{#1}}
\expandafter\ifx\csname urlstyle\endcsname\relax
  \providecommand{\doi}[1]{doi: #1}\else
  \providecommand{\doi}{doi: \begingroup \urlstyle{rm}\Url}\fi

\bibitem[Bakry \& {\'E}mery(2006)Bakry and {\'E}mery]{bakry2006diffusions}
Bakry, D. and {\'E}mery, M.
\newblock Diffusions hypercontractives.
\newblock In \emph{S{\'e}minaire de Probabilit{\'e}s XIX 1983/84: Proceedings},
  pp.\  177--206. Springer, 2006.

\bibitem[Bakry et~al.(2008)Bakry, Barthe, Cattiaux, and
  Guillin]{bakry2008simple}
Bakry, D., Barthe, F., Cattiaux, P., and Guillin, A.
\newblock A simple proof of the poincar{\'e} inequality for a large class of
  probability measures.
\newblock \emph{Electronic Communications in Probability}, 13:\penalty0 60--66,
  2008.

\bibitem[Bakry et~al.(2014)Bakry, Gentil, Ledoux, et~al.]{bakry2014analysis}
Bakry, D., Gentil, I., Ledoux, M., et~al.
\newblock \emph{Analysis and geometry of Markov diffusion operators}, volume
  103.
\newblock Springer, 2014.

\bibitem[Balasubramanian et~al.(2022)Balasubramanian, Chewi, Erdogdu, Salim,
  and Zhang]{balasubramanian2022towards}
Balasubramanian, K., Chewi, S., Erdogdu, M.~A., Salim, A., and Zhang, S.
\newblock Towards a theory of non-log-concave sampling: first-order
  stationarity guarantees for langevin monte carlo.
\newblock In \emph{Conference on Learning Theory}, pp.\  2896--2923. PMLR,
  2022.

\bibitem[Barkhagen et~al.(2021)Barkhagen, Chau, Moulines, R{\'a}sonyi, Sabanis,
  and Zhang]{convex}
Barkhagen, M., Chau, N.~H., Moulines, {\'E}., R{\'a}sonyi, M., Sabanis, S., and
  Zhang, Y.
\newblock On stochastic gradient langevin dynamics with dependent data streams
  in the logconcave case.
\newblock \emph{Bernoulli}, 27\penalty0 (1):\penalty0 1--33, 2021.

\bibitem[Bolley \& Villani(2005)Bolley and Villani]{bolley2005weighted}
Bolley, F. and Villani, C.
\newblock Weighted csisz{\'a}r-kullback-pinsker inequalities and applications
  to transportation inequalities.
\newblock In \emph{Annales de la Facult{\'e} des sciences de Toulouse:
  Math{\'e}matiques}, volume~14, pp.\  331--352, 2005.

\bibitem[Brosse et~al.(2019)Brosse, Durmus, Moulines, and Sabanis]{tula}
Brosse, N., Durmus, A., Moulines, {\'E}., and Sabanis, S.
\newblock The tamed unadjusted {L}angevin algorithm.
\newblock \emph{Stochastic Processes and their Applications}, 129\penalty0
  (10):\penalty0 3638--3663, 2019.

\bibitem[Cattiaux et~al.(2013)Cattiaux, Guillin, and
  Zitt]{cattiaux2013poincare}
Cattiaux, P., Guillin, A., and Zitt, P.~A.
\newblock Poincar{\'e} inequalities and hitting times.
\newblock In \emph{Annales de l'IHP Probabilit{\'e}s et statistiques},
  volume~49, pp.\  95--118, 2013.

\bibitem[Chau et~al.(2021)Chau, Moulines, R{\'a}sonyi, Sabanis, and
  Zhang]{nonconvex}
Chau, N.~H., Moulines, {\'E}., R{\'a}sonyi, M., Sabanis, S., and Zhang, Y.
\newblock On stochastic gradient langevin dynamics with dependent data streams:
  The fully nonconvex case.
\newblock \emph{SIAM Journal on Mathematics of Data Science}, 3\penalty0
  (3):\penalty0 959--986, 2021.

\bibitem[Cheng et~al.(2018)Cheng, Chatterji, Abbasi-Yadkori, Bartlett, and
  Jordan]{berkeley}
Cheng, X., Chatterji, N.~S., Abbasi-Yadkori, Y., Bartlett, P.~L., and Jordan,
  M.~I.
\newblock Sharp convergence rates for {L}angevin dynamics in the nonconvex
  setting.
\newblock \emph{arXiv preprint arXiv:1805.01648}, 2018.

\bibitem[Chewi et~al.(2021)Chewi, Erdogdu, Li, Shen, and
  Zhang]{chewi2021analysis}
Chewi, S., Erdogdu, M.~A., Li, M.~B., Shen, R., and Zhang, M.
\newblock Analysis of langevin monte carlo from poincar$\backslash$'e to
  log-sobolev.
\newblock \emph{arXiv preprint arXiv:2112.12662}, 2021.

\bibitem[Dalalyan(2017)]{dalalyan2017theoretical}
Dalalyan, A.~S.
\newblock Theoretical guarantees for approximate sampling from smooth and
  log-concave densities.
\newblock \emph{Journal of the Royal Statistical Society: Series B (Statistical
  Methodology)}, 79\penalty0 (3):\penalty0 651--676, 2017.

\bibitem[Durmus \& Moulines(2017)Durmus and Moulines]{durmus2017nonasymptotic}
Durmus, A. and Moulines, E.
\newblock Nonasymptotic convergence analysis for the unadjusted {L}angevin
  algorithm.
\newblock \emph{The Annals of Applied Probability}, 27\penalty0 (3):\penalty0
  1551--1587, 2017.

\bibitem[Durmus \& Moulines(2019)Durmus and Moulines]{durmus2019high}
Durmus, A. and Moulines, E.
\newblock High-dimensional {B}ayesian inference via the unadjusted {L}angevin
  algorithm.
\newblock \emph{Bernoulli}, 25\penalty0 (4A):\penalty0 2854--2882, 2019.

\bibitem[Erdogdu \& Hosseinzadeh(2021)Erdogdu and
  Hosseinzadeh]{erdogdu2021convergence}
Erdogdu, M.~A. and Hosseinzadeh, R.
\newblock On the convergence of langevin monte carlo: The interplay between
  tail growth and smoothness.
\newblock In \emph{Conference on Learning Theory}, pp.\  1776--1822. PMLR,
  2021.

\bibitem[Erdogdu et~al.(2022)Erdogdu, Hosseinzadeh, and
  Zhang]{erdogdu2022convergence}
Erdogdu, M.~A., Hosseinzadeh, R., and Zhang, S.
\newblock Convergence of langevin monte carlo in chi-squared and r{\'e}nyi
  divergence.
\newblock In \emph{International Conference on Artificial Intelligence and
  Statistics}, pp.\  8151--8175. PMLR, 2022.

\bibitem[Hutzenthaler et~al.(2011)Hutzenthaler, Jentzen, and
  Kloeden]{hutzenthaler2011}
Hutzenthaler, M., Jentzen, A., and Kloeden, P.~E.
\newblock Strong and weak divergence in finite time of euler{\textquoteright}s
  method for stochastic differential equations with non-globally lipschitz
  continuous coefficients.
\newblock \emph{Proceedings of the Royal Society of London A: Mathematical,
  Physical and Engineering Sciences}, 467\penalty0 (2130):\penalty0 1563--1576,
  2011.
\newblock ISSN 1364-5021.

\bibitem[Hutzenthaler et~al.(2012)Hutzenthaler, Jentzen, and
  Kloeden]{hutzenthaler2012}
Hutzenthaler, M., Jentzen, A., and Kloeden, P.~E.
\newblock Strong convergence of an explicit numerical method for sdes with
  nonglobally lipschitz continuous coefficients.
\newblock \emph{Ann. Appl. Probab.}, 22\penalty0 (4):\penalty0 1611--1641, 08
  2012.

\bibitem[Johnston et~al.(2023)Johnston, Lytras, and
  Sabanis]{johnston2023kinetic}
Johnston, T., Lytras, I., and Sabanis, S.
\newblock Kinetic langevin mcmc sampling without gradient lipschitz
  continuity--the strongly convex case.
\newblock \emph{arXiv preprint arXiv:2301.08039}, 2023.

\bibitem[Li \& Erdogdu(2020)Li and Erdogdu]{li2020riemannian}
Li, M.~B. and Erdogdu, M.~A.
\newblock Riemannian langevin algorithm for solving semidefinite programs.
\newblock \emph{arXiv preprint arXiv:2010.11176}, 2020.

\bibitem[Lovas et~al.(2023)Lovas, Lytras, R{\'a}sonyi, and Sabanis]{TUSLA}
Lovas, A., Lytras, I., R{\'a}sonyi, M., and Sabanis, S.
\newblock Taming neural networks with tusla: Nonconvex learning via adaptive
  stochastic gradient langevin algorithms.
\newblock \emph{SIAM Journal on Mathematics of Data Science}, 5\penalty0
  (2):\penalty0 323--345, 2023.

\bibitem[Lytras \& Sabanis(2023)Lytras and Sabanis]{lytras2023taming}
Lytras, I. and Sabanis, S.
\newblock Taming under isoperimetry.
\newblock \emph{arXiv preprint arXiv:2311.09003}, 2023.

\bibitem[Majka et~al.(2020)Majka, Mijatovi{\'c}, and
  Szpruch]{majka2020nonasymptotic}
Majka, M.~B., Mijatovi{\'c}, A., and Szpruch, {\L}.
\newblock Nonasymptotic bounds for sampling algorithms without log-concavity.
\newblock \emph{The Annals of Applied Probability}, 30\penalty0 (4):\penalty0
  1534--1581, 2020.

\bibitem[McNabb(1986)]{mcnabb1986comparison}
McNabb, A.
\newblock Comparison theorems for differential equations.
\newblock \emph{Journal of mathematical analysis and applications},
  119\penalty0 (1-2):\penalty0 417--428, 1986.

\bibitem[Menz \& Schlichting(2014)Menz and Schlichting]{menz2014poincare}
Menz, G. and Schlichting, A.
\newblock Poincar{\'e} and logarithmic sobolev inequalities by decomposition of
  the energy landscape.
\newblock \emph{The Annals of Probability}, 42\penalty0 (5):\penalty0
  1809--1884, 2014.

\bibitem[Mou et~al.(2022)Mou, Flammarion, Wainwright, and
  Bartlett]{mou2022improved}
Mou, W., Flammarion, N., Wainwright, M.~J., and Bartlett, P.~L.
\newblock Improved bounds for discretization of langevin diffusions:
  Near-optimal rates without convexity.
\newblock \emph{Bernoulli}, 28\penalty0 (3):\penalty0 1577--1601, 2022.

\bibitem[Mousavi-Hosseini et~al.(2023)Mousavi-Hosseini, Farghly, He,
  Balasubramanian, and Erdogdu]{mousavi2023towards}
Mousavi-Hosseini, A., Farghly, T.~K., He, Y., Balasubramanian, K., and Erdogdu,
  M.~A.
\newblock Towards a complete analysis of langevin monte carlo: Beyond
  poincar{\'e} inequality.
\newblock In \emph{The Thirty Sixth Annual Conference on Learning Theory}, pp.\
   1--35. PMLR, 2023.

\bibitem[Neufeld et~al.(2022)Neufeld, En, and Zhang]{neufeld2022non}
Neufeld, A., En, M. N.~C., and Zhang, Y.
\newblock Non-asymptotic convergence bounds for modified tamed unadjusted
  langevin algorithm in non-convex setting.
\newblock \emph{arXiv preprint arXiv:2207.02600}, 2022.

\bibitem[Nguyen et~al.(2021)Nguyen, Dang, and Chen]{nguyen2021unadjusted}
Nguyen, D., Dang, X., and Chen, Y.
\newblock Unadjusted langevin algorithm for non-convex weakly smooth
  potentials.
\newblock \emph{arXiv preprint arXiv:2101.06369}, 2021.

\bibitem[Raginsky et~al.(2017)Raginsky, Rakhlin, and Telgarsky]{raginsky}
Raginsky, M., Rakhlin, A., and Telgarsky, M.
\newblock Non-convex learning via {S}tochastic {G}radient {L}angevin
  {D}ynamics: a nonasymptotic analysis.
\newblock In \emph{Conference on Learning Theory}, pp.\  1674--1703, 2017.

\bibitem[Sabanis(2013)]{tamed-euler}
Sabanis, S.
\newblock A note on tamed euler approximations.
\newblock \emph{Electron. Commun. Probab.}, 18\penalty0 (47):\penalty0 1--10,
  2013.

\bibitem[Sabanis(2016)]{SabanisAoAP}
Sabanis, S.
\newblock Euler approximations with varying coefficients: the case of
  superlinearly growing diffusion coefficients.
\newblock \emph{Ann. Appl. Probab.}, 26\penalty0 (4):\penalty0 2083--2105,
  2016.

\bibitem[Vempala \& Wibisono(2019)Vempala and Wibisono]{vempala2019rapid}
Vempala, S. and Wibisono, A.
\newblock Rapid convergence of the unadjusted langevin algorithm: Isoperimetry
  suffices.
\newblock \emph{Advances in neural information processing systems}, 32, 2019.

\bibitem[Zhang et~al.(2023)Zhang, Akyildiz, Damoulas, and
  Sabanis]{zhang2023nonasymptotic}
Zhang, Y., Akyildiz, {\"O}.~D., Damoulas, T., and Sabanis, S.
\newblock Nonasymptotic estimates for stochastic gradient langevin dynamics
  under local conditions in nonconvex optimization.
\newblock \emph{Applied Mathematics \& Optimization}, 87\penalty0 (2):\penalty0
  25, 2023.

\end{thebibliography}

\end{document}